\documentclass[11pt]{article}
\usepackage[a4paper,left=0.95in,right=0.95in,top=1in,bottom=1in]{geometry}
\usepackage{microtype}
\usepackage{subcaption} 
\usepackage{graphicx}
\usepackage{booktabs} % for professional tables
\usepackage{hyperref}
\usepackage{cuted}
\usepackage{amsmath}
\usepackage{amssymb}
\usepackage{mathtools}
\usepackage{amsthm}
\usepackage{braket}
\usepackage{bm}
\usepackage{xcolor}
\usepackage[numbers,sort&compress]{natbib}
\usepackage[capitalize,noabbrev]{cleveref}
\usepackage{enumerate}
\usepackage{enumitem}
\usepackage{thmtools} % for re-stating propositions / theorems / etc.
\usepackage{thm-restate}
\theoremstyle{plain}
\newtheorem{theorem}{Theorem}[section]
\newtheorem{proposition}[theorem]{Proposition}
\newtheorem{lemma}[theorem]{Lemma}
\newtheorem{corollary}[theorem]{Corollary}
\theoremstyle{definition}
\newtheorem{definition}[theorem]{Definition}
\newtheorem{assumption}[theorem]{Assumption}
\theoremstyle{remark}
\newtheorem{remark}[theorem]{Remark}
\newcommand{\bR}{\mathbb{R}}
\newcommand{\bE}{\mathbb{E}}
\newcommand{\bN}{\mathbb{N}}
\newcommand{\cD}{\mathcal{D}}
\newcommand{\cG}{\mathcal{G}}
\newcommand{\cH}{\mathcal{H}}
\newcommand{\cN}{\mathcal{N}}
\newcommand{\cX}{\mathcal{X}}
\newcommand{\cY}{\mathcal{Y}}
\newcommand{\Cov}{\Sigma} 
    
\DeclareMathOperator{\Softmax}{Softmax}
\DeclareMathOperator{\Spec}{Spec}

\DeclareMathOperator{\LayerNorm}{LN}
\DeclareMathOperator{\mtc}{MC}
\DeclareMathOperator{\mlp}{MLP}

\DeclareMathOperator{\emb}{emb}

\DeclareMathOperator{\dist}{dist}
\DeclareMathOperator{\diag}{diag}

\DeclareMathOperator{\flops}{FLOPs}

\DeclareMathOperator{\Conf}{Conf}
\DeclareMathOperator*{\argmin}{argmin}

\newcommand{\abs}[1]{\left\lvert#1\right\rvert}
\newcommand{\norm}[1]{\left\lVert#1\right\rVert}
\newcommand{\dmod}{d_{\text{model}}}

\newcommand{\bfx}{\bm{x}}
\newcommand{\bfz}{\bm{z}}
\DeclareRobustCommand{\HronAttentionCitationI}{\citep[Theorem~3(I)]{Hron2020}}
\DeclareRobustCommand{\HronAttentionCitationII}{\citep[Theorem~3(II)]{Hron2020}}
\definecolor{orange}{rgb}{1,0.5,0}
\definecolor{deepgreen}{RGB}{0,168,107}
\title{Algorithmic Task Capture, Computational Complexity, and Inductive Bias of Infinite Transformers}
\date{}
\begin{document}
\maketitle
\author{
\makebox[\textwidth][c]
{\begin{tabular}{cc}
    Orit Davidovich & Zohar Ringel \\
    IBM Research, & Racah Institute of Physics, \\
    Haifa, Israel & Hebrew University, Jerusalem, Israel \\
    \texttt{orit.davidovich@ibm.com} & \texttt{zohar.ringel@mail.huji.ac.il}
\end{tabular}
}}
%%%%%%%%%%%%%%%%%%%%%%%%%%%%%%%%
% ABSTRACT
%%%%%%%%%%%%%%%%%%%%%%%%%%%%%%%%
\begin{abstract}
We formally define algorithmic capture of combinatorial tasks as the ability of a transformer to extrapolate to arbitrary task sizes with controllable error and logarithmic sample adaptation, providing a sharp scaling criterion for distinguishing logic internalization from statistical interpolation. Empirically, across scaling ranges spanning up to 2.5 orders of magnitude, we observe evidence of capture and non-capture. By analyzing infinite-width transformers in both the lazy and rich regimes, we derive upper bounds on the inference-time computational complexity of the combinatorial tasks these networks can capture. We show that, despite their universal expressivity, transformers possess an inductive bias that disfavors higher-complexity algorithmic procedures within the efficient polynomial-time heuristic scheme class, consistent with successful capture on simpler combinatorial tasks such as induction heads, sort, and string matching.
\end{abstract}
%%%%%%%%%%%%%%%%%%%%%%%%%%%%%%%%
% KEYWORDS
%%%%%%%%%%%%%%%%%%%%%%%%%%%%%%%%

\section{Introduction}

A central yet unresolved question in the study of large language models (LLMs) is the extent to which the procedures they execute involve genuine ``understanding'' or merely exploit statistical correlations to interpolate a domain.   Empirical work has often highlighted the fragility of this ``understanding"; for instance, the GSM-symbolic benchmark \citep{Mirzadeh2024} shows that performance on mathematical reasoning tasks can degrade significantly when symbolic templates are altered, suggesting a reliance on pattern matching rather than robust algorithmic execution. Still, the loose and somewhat philosophical nature of these terms hinders rigorous progress. One area where statistical learning can be clearly discerned from logic internalization is that of combinatorial tasks (Sec.~\ref{sec:combinatorial_setting}). Here, extrapolation to arbitrarily-large task instance sizes is an inherent requirement. Such strong out-of-distribution (OOD) settings make statistical correlation and simple interpolation far less effective \citep{Xu2020NAR}. Furthermore, existing algorithms provide a reliable ground truth regarding worst-case, average-case, and heuristic complexity. These theoretical bounds can be contrasted with the computation required for network inference, allowing us to refute the possibility that a network ``grokked'' \citep{Power2022Grokking} a task without needing to interpret its inner workings. Finally, transformers -- the core technology behind LLMs -- can operate on variable sequence lengths, naturally lending themselves to this scalable setting once we align instance size with context length $T$. 

\begin{table*}[t]
\centering
\small
\begin{tabular}{llccc}
\hline \noalign{\vskip 2pt}
\textbf{Regime} & \textbf{Evaluation}
& \textbf{\parbox[t]{2cm}{\raggedright Inference Complexity}} 
& \textbf{\parbox[t]{2cm}{\raggedright Evaluation Error}} 
& \textbf{EPTHS} \\[1ex]
\noalign{\vskip 2pt} \hline \noalign{\vskip 1ex}
{NNGP/NTK} &  Kernel estimate &  $O(PN_{\mtc}T^3)$ & $\sqrt{P/N_{\mtc}}$  &  $O(T^{3+\epsilon})$ \\[1ex]
 &  Forward-pass & $O(NT^2)$ &  $P^{\gamma} / N$ & $O(T^{2+\epsilon})$
\\ [1ex]
\parbox[t]{4cm}{\raggedright Rich Mean-Field Scaling} &  Forward-pass & $O(NT^2)$ &  $P^{\gamma} / N$ & $O(T^{2+\epsilon})$ \\[1ex]
\hline
\end{tabular}
\caption{Summary of  theoretical results. For infinite lazy transformers, the kernel-predictor (NNGP/NTK) can be estimated using Monte Carlo with the specified inference complexity and error. Consequently, tasks with EPTHS complexity larger than $T^3$ times log factors cannot be captured ($T^{3+\epsilon}$). Approximating this predictor by a forward-pass on an extensive width network ($N \propto P^{\gamma \ge  1}$), we obtain an improved bound on the complexity of capturable tasks under lazy learning. Making the reasonable assumption that extensive networks with rich scaling converge to their rich/mean-field limits, we obtain the same complexity bound even under strong feature learning.} 

%Summary of infinite-width transformer inference-time complexity dependent on the context length $T$; the number of training datapoints $P$; the number of MC samples $N_{\mtc}$; the width of the estimating finite-width transformer $N$; a $T$-independent constant $\gamma>0$; and $\epsilon>0$. The forward-pass is relevant for large but finite-width transformers. The kernel estimate utilizes MC integration.}
\label{tab:complexity}
\end{table*}

In this work, we analyze whether the inductive biases of transformers inhibit or promote combinatorial task understanding. We purposely set aside the potential yet not at all obvious \citep{Dong2021Attention, Gromov2024Ineffectiveness} benefits of scaling depth and scratchpads to focus on the complexity of next-token predictions. To enable arbitrarily complex functions and to remove trivial computational bottlenecks associated with finite width, we consider infinitely over-parameterized networks in both the lazy and rich training regimes with an extensive scaling of width to problem size. Our key contributions are as follows.
\begin{itemize}[leftmargin=3mm]
    \item \textbf{Formal Definition of Combinatorial Task Capture:} We provide a verifiable definition of what it means for a neural network to capture a combinatorial task (Definition~\ref{def:capture}), involving an input measure that is scalable across problem sizes ($\mu_{X,T}$), an arbitrary but finite sample budget $P_0(\delta)$ to obtain correct task instance outcomes with probability ($1-\delta$) on instances up to size $T_0$, and a small logarithmic budget ($O(\log(T/T_0))$) for fine-tuning on larger instance sizes ($T>T_0$). 
    % This logarithmic allowance is provided solely to correct for architectural non-idealities rather than to learn the task logic itself. 
    \item \textbf{Empirical Evidence of Captured and Non-Captured Tasks:} We identify transformers that capture the induction head, sorting, and string matching tasks. Conversely, we show several transformers, including very deep ones, that fail to capture the source-target shortest path problem, the Max Flow / Minimal Cut problem, and the context-free grammar parsing problem (Section~\ref{sec:experiments}). 
    \item \textbf{Heuristic Complexity:} We introduce the efficient polynomial-time heuristic scheme (EPTHS), a distributional complexity class for combinatorial tasks that formalizes the existence of $\delta$-correct algorithms with runtime polynomial in the instance size $T$ (Definition~\ref{def:epths}). 
    % EPTHS bridges approximation-scheme notions (e.g., EPTAS) and average-case complexity by replacing worst-case guarantees with high-probability exact recovery.
    This gives rise to a task-level notion of complexity aligned with learning-based solvers and provides a principled scale against which inference-time procedures can be compared.
    \item \textbf{Upper Complexity Bounds:} We show that the kernel predictor of an infinite-width transformer can be estimated in $O(PN_{\mtc}T^3)$ FLOPs with $P$ datapoints and $N_{\mtc}$ Monte Carlo (MC) samples (Theorem~\ref{thm:TransformerPropagationComplexity}). Consequently, we prove that an infinite-width transformer cannot capture a combinatorial task of EPTHS complexity above $O(T^{3+\epsilon})$ (Corollary~\ref{coro:epths_for_kernel_evaluation}); under a finite-width convergence assumption for lazy and rich infinite-width limits (Assumption~\ref{assm:discrepancy_scale}), this improves to $O(T^{2+\epsilon})$ (Corollary~\ref{coro:epths_for_feed_forward}). Table~\ref{tab:complexity} summarizes these bounds. This suggests that, while feature learning can have drastic sample complexity effects \citep{Bietti2022, Wei2019, Ghorbani2021}, it does not have similar inference complexity effects. Our theoretical results are summarized in Table~\ref{tab:complexity}.
\end{itemize}

Our main purpose in this work is to present a framework in which computational complexity theory, OOD inductive bias of transformers, and shortcut learning can be sharply defined and contrasted. Further refinement of the upper complexity bounds is needed to explain some of our empirical results.  

%%%%%%%%%%%%%%%%%%%%%%%%%%%%%%%%
%%%%%%%%%%%%%%%%%%%%%%%%%%%%%%%%
%%%%%%%%%%%%%%%%%%%%%%%%%%%%%%%%

\section{Related Work}

Our work intersects with several active areas of deep learning and complexity theory.

\textbf{Grokking an Algorithm:}
The phenomenon of ``grokking'' -- where generalization emerges long after overfitting -- has been widely observed in math tasks \citep{Power2022Grokking,Barak2022,Nanda2023}. However, these tasks have $O(\log(n)),O(1)$ computational complexity and, more crucially, experiments were carried out on specific instance sizes, rendering computational complexity ill-defined.

%\textbf{Infinite-Width Limits (NTK/NNGP):}
%The behavior of neural networks in the infinite-width limit is well-characterized by the Neural Tangent Kernel (NTK) \citep{Jacot2018}. Extensions of this theory to transformers \citep{Hron2020} show that self-attention layers converge to specific kernel processes. We build upon these foundations, moving beyond convergence analysis to bound the \textit{computational complexity} of the resulting predictors.

%\textbf{Feature Learning vs. Kernel Methods:}
%A major distinction in modern theory is between the ``lazy'' regime (kernel learning) and the ``rich'' regime (feature learning) \citep{Rubin2025, Bordelon2022}. While feature learning is crucial for sample efficiency in structured tasks \citep{Aitchison2020, Sompolinsky1988}, our results indicate that the inference-time computational constraints of the transformer architecture remain a bottleneck even in the rich regime \citep{Menard2023}.

\textbf{Kolmogorov and Sample Complexity.} Prior work in algorithmic information theory suggests that neural networks are strongly biased towards functions with low Kolmogorov complexity \citep{VallePerez2018, Dingle2018}. This is a different notion of complexity, which, for functions representing an algorithm, does not scale with input size. In addition, various works \citep{Bietti2022,Ghorbani2021,Wei2019} studied the effect of feature learning on sample complexity. \citet{Lavie2024} study the sample complexity of infinite transformers. However, functions with high sample complexity are not necessarily hard to compute (e.g. $k$-sparse parity). 

\textbf{Worst-case Expressivity / Circuit Complexity.} A complementary line of work studies transformers as constant-depth threshold circuits, placing them in $\mathsf{TC}^0$ \citep{Merrill2022Saturated,Chiang2023Tighter}. More recently, \citet{Chen2024Theoretical} established an unconditional lower bound for multi-layer decoder-only transformers. In contrast, our perspective is ``softer''. We analyze average-case behavior and learnability via infinite-width limits (NNGP/NTK), where success is approximate and often stated in high probability, and where complexity is tied to inference cost / kernel evaluation rather than to exact worst-case computation. In the limit, we have universal expressivity, and thus focus on inductive bias, whereas the aforementioned literature addresses limits to expressivity.

\textbf{Length Generalization.} 
A growing body of work studies whether transformers extrapolate to longer inputs than seen during training, particularly on algorithmic and reasoning tasks. 
One line of work characterizes length generalization through transformer-aligned computational models: RASP provides a formal language for transformer-computable programs~\citep{weiss2021thinkingliketransformers}, and \citet{zhou2024what} conjecture length generalization when the target task admits a sufficiently simple RASP-L program. 
Our notion of algorithmic capture is instead task-level and black-box with respect to any particular transformer program. 
A second line of work studies training procedures and implicit biases, showing that standard training often fails to length-generalize on tasks such as parity or addition, while scratchpads, chain-of-thought (CoT), or curriculum learning can improve extrapolation~\citep{abbe2024generalization,Anil2022LengthGeneralization}. 
Beyond fixed-horizon extrapolation, our formulation tracks the scaling of the adaptation budget with the context horizon and extends to combinatorial optimization tasks -- including NP-hard problems with approximation schemes, such as the Euclidean traveling saleman problem -- where failure can be related to complexity-theoretic obstructions.

\textbf{Neural Algorithmic Reasoning (NAR).} 
Neural algorithmic reasoning (NAR) has provided an empirical and conceptual foundation for studying whether neural networks can execute classical algorithms, through work on algorithmic alignment and extrapolation \citep{Xu2020NAR,Xu2021NAR}, methods based on causal regularization \citep{Bevilacqua2023NAR}, no-hint and self-supervised training \citep{Rodionov2023NAR}, deep equilibrium formulations \citep{Georgiev2024NAR}, and newer dynamic-programming-inspired architectures such as FloydNet \citep{Yu2026NAR}. 
Complementing this line of work, we ask a task-level question: whether a transformer trained under standard supervision can maintain performance on the combinatorial problem itself as instance size grows. 
Our logarithmic re-adaptation criterion turns this into a black-box asymptotic test, closer to the evaluation perspective of combinatorial optimization than to imitation of a fixed canonical execution trace.

% Our work is complementary, but asks a different question. Rather than primarily evaluating whether a model can imitate the execution of a prescribed algorithm on held-out larger instances, often with intermediate supervision or other algorithm-specific priors, we study algorithmic combinatorial task capture under standard supervision. While NAR is often studied under evaluation protocols that permit some loss in OOD performance, we instead ask whether a transformer learns a scalable solution procedure for the combinatorial problem itself, with performance recoverable under growth in instance size using only a logarithmic re-adaptation budget. This shift reflects an application-oriented perspective closer to combinatorial optimization, where the central object is the problem to be solved at scale rather than a particular canonical execution trace. In that sense, our notion of capture is not intended to negate the NAR program, but to sharpen the distinction between partial size extrapolation and learning a scalable task-level heuristic.

%%%%%%%%%%%%%%%%%%%%%%%%%%%%%%%%
%%%%%%%%%%%%%%%%%%%%%%%%%%%%%%%%
%%%%%%%%%%%%%%%%%%%%%%%%%%%%%%%%

\section{Problem Formulation and Main Definitions}
\label{sec:setting}

Here we introduce our two key notions: combinatorial task capture (Definition~\ref{def:capture}) and efficient polynomial-time heuristic schemes (Definition~\ref{def:epths}). As noted in the introduction, we will show that, if an infinite-width transformer algorithmically captures a combinatorial task, then its kernel predictor places that task in an EPTHS class of complexity $O(T^{k+\epsilon})$, with $k=3$ for any $\epsilon>0$ (Sec.~\ref{sec:kernel_evaluation_complexity}), and down to $k=2$ under additional assumptions (Sec.~\ref{sec:tighter_bounds}). Consequently, heuristic complexity yields an obstruction to capture: if the transformer's inference-time heuristic complexity grows strictly slower than that of the combinatorial task, then the latter cannot be captured.

%%%%%%%%%%%%%%%%%%%%%%%%%%%%%%%%
%%%%%%%%%%%%%%%%%%%%%%%%%%%%%%%%
%%%%%%%%%%%%%%%%%%%%%%%%%%%%%%%%

\subsection{Combinatorial Task}
\label{sec:combinatorial_setting}

In this work, we consider supervised learning of combinatorial tasks with discrete outputs. For the purposes of this paper, we use the following broad problem-oriented notion. A {\it combinatorial task} $g: \cX \to \cY$ is a map from a space of task instances $\cX$ to an ambient space of possible admissible outcomes $\cY$. To each instance $X \in \cX$, there is an associate discrete configuration set $\Conf(X)$ of admissible solutions and a score function $\psi_X: \Conf(X) \to \bR$. 
The combinatorial task is in {\it solution form} if $g(X) \in \Conf(X) \subseteq  \cY$, where $\psi_X(g(X)) = \min_{\rho \in \Conf(X)} \psi_X(\rho)$, meaning $g(X)$ is a distinguished minimizer, and in {\it value form} if $g(X) =\min_{\rho \in \Conf(X)} \psi_X(\rho) \in \cY \subseteq \bR$. We assume that $\cY$ is equipped with a task-specific metric $\dist_{\cY}$ so that $\Omega = \inf\left\{\dist_{\cY}(\rho,\rho') \mid \rho \neq \rho'\in\Conf(X)\subseteq  \cY, X\in\cX\right\}>0$,
and there is an {\it accuracy margin} $\Delta \ge \Omega > 0$
large enough that any prediction within $\Delta/2$ of the target implies recovery of a desired outcome. Finally, we assume a notion of task size $T = T(X) \in \bN$ which stratifies $\cX = \bigsqcup_{T \ge 0} \cX_T$ into instance sets $\cX_T$ per size $T$. 
As an example, consider the shortest path problem (see App.~\ref{app:combinatorial_tasks} for details). Here, $\cX$ is a set of finite undirected graphs with distinguished source and target nodes and $X$ is one such. The configuration set $\Conf(X)$ consists of all paths from the source to the target. The task size $T(X)=V(X)$ is the number of vertices in the graph $X$. The task $g$ will be in solution form if $g(X)$ is a path of minimal length, while it will be in value form when $g(X)$ is that length. 

% One suitable metric in solution form is $\dist_{\cY}(\rho,\rho') = \abs{\psi_X(\rho)-\psi_X(\rho')}+c\bm{1}_{\rho\neq\rho'}$ for some $c>0$ and $\rho,\rho'\in\Conf(X),X\in\cX$. Here, $\Omega=c$, and we can choose $\Delta = 4c$ so that our accuracy level does not distinguish between two minimal paths.

%%%%%%%%%%%%%%%%%%%%%%%%%%%%%%%%
%%%%%%%%%%%%%%%%%%%%%%%%%%%%%%%%
%%%%%%%%%%%%%%%%%%%%%%%%%%%%%%%%

\subsection{Transformer Setting}
\label{sec:transformer_setting}

In the transformer setting studied here, we represent an instance $X \in \cX$ of a combinatorial task of size $T=T(X)$ by an embedded array
\begin{align}
    X = (\bfx_1,\ldots,\bfx_T) \in \cX_T \subseteq \bR^{T\times d}, \quad \bfx_a \in \bR^d, \quad a \in [T]
\end{align}
where $d$ is the token embedding dimension and $T$ plays the role of the context length and task size. Thus, throughout this paper, instances are already viewed in their embedded representation rather than as raw symbolic objects, and $\cY$ is chosen accordingly.
We further assume that, for each size $T$, the instance space $\cX_T \subseteq \bR^{T\times d}$ is equipped with a probability measure $\mu_{X,T}$ from which training and test instances can be efficiently sampled. Where appropriate, for example, for graph combinatorial tasks, $\mu_{X,T}$ will be chosen in the critical regime, e.g., where large graph clusters emerge.

In general, $X$ is not assumed permutable, i.e., $X_\sigma = (\bfx_{\sigma(1)},\ldots,\bfx_{\sigma(T)})$, $\sigma\in S_T$, may represent a different task instance. In such a case we rely on positional encoding (PE). For combinatorial tasks in which all sequence indices in $X$ are permutable, but for $O(1)$ special tokens -- for example, those representing the source and target nodes in the finite undirected graph of a shortest path problem -- we encode such large symmetry using PE that only distinguishes between special tokens and the rest. 

Architecture-wise, we focus on decoder-only transformers with softmax self-attention \citep{Vaswani2017Transformer}. To establish notation, consider a single attention layer of a transformer block with $H$ heads. It takes as inputs $\bm{z'}_a(X) \in \bR^{\dmod}$, $a \in [T]$, which we package into $Z'(X) \in \bR^{T \times \dmod}$. The attention layer's output at token $a \in [T]$ is then
\begin{equation}
    \bm{z}_a(X) = \frac{1}{\sqrt{H}} \sum_{h=1}^H \sum_{c=1}^T A^h_{ac}(X) W^{h\top}_{V} \bm{z'}_c(X),
\end{equation}
where
\begin{align}
    A^h_{ac}(X) = \Softmax\left( S_{a\bullet}^h(X) \right)_c
\end{align}
are the attention weights and
\begin{align}
\label{eq:attention_scores}
    S^{h}_{ac}(X)=\frac{1}{\sqrt{d_k}} {\bm{z'}_a}(X)^{\top} W^h_Q (W^h_K)^{\top} \bm{z'}_{c}(X)
\end{align}
are the attention scores for query and key weight matrices $W_K^h, W_Q^h \in \bR^{\dmod \times d_k}$ and value weight matrices $W_V^h \in \bR^{\dmod \times \dmod}$ (see App.~\ref{app:transformer_setup} for further details). 

%%%%%%%%%%%%%%%%%%%%%%%%%%%%%%%%
%%%%%%%%%%%%%%%%%%%%%%%%%%%%%%%%
%%%%%%%%%%%%%%%%%%%%%%%%%%%%%%%%

\subsection{Task Capture and EPTHS}
\label{sec:definitions}

Now we are ready to rigorously define our two key notions: the algorithmic capture of a combinatorial task (Definition~\ref{def:capture}) and efficient polynomial-time heuristic schemes (Definition~\ref{def:epths}). 

\begin{definition}[Algorithmic capture of a combinatorial task]
\label{def:capture}
A transformer $f$ is said to \textit{algorithmically capture} a combinatorial task $g: \cX \to \cY$ if for every sufficiently small $\delta > 0$ there exist $T_0=T_0(\delta)$, $C_0=C_0(\delta)$, and $P_0=P_0(\delta)$ such that, for every $T\ge T_0$,
\begin{align}
    \Pr\nolimits_{X \sim \mu_{X,T}}\left[\dist_{\cY}(f_T(X),g(X))<\Delta/3\right] > 1-\delta,
\end{align}
where $f_T$ denotes the predictor obtained by the following two-stage training protocol: (1) initial training on $P_0$ samples from $\mu_{X,1},\dots,\mu_{X,T_0}$ shared across all $T\ge T_0$; and (2) a secondary fine-tuning on $C_0\log(T/T_0)$ additional samples from $\mu_{X,1}, \dots, \mu_{X,T}$.\footnote{The underlying transformer architecture, training procedure, hyperparameter choices, and initialization/training randomness convention are fixed; the probability displayed is over the test instance $X\sim\mu_{X,T}$. For a PAC-style extension of Definition~\ref{def:capture}, see App.~\ref{app:pac_style_capture}.}  
\end{definition}

% Alternatively, we allow $M$ separate fine-tuning stages with $T_0 < T_1 < T_2 \dots < T_M=T$ with $C\log(T_{m}/T_{m-1})$ additional data-points sampled from $\mu_{X,1}, \dots, \mu_{X,T_{m}}$. 

The choice of logarithmic data adaptation in Definition~\ref{def:capture} reflects a compromise between the ideal of perfect extrapolation and the practical non-idealities of transformers. In principle, one might regard $O(1)$ additional samples as the cleanest signature of genuine capture. 
% if the task logic has already been learned up to size $T_0$, then no substantial new supervision should be needed for larger $T$. 
Empirically, however, this standard appears too strict. In combinatorial optimization problems, for instance, performance typically deteriorates as one moves away from the training size distribution; see, for instance, \citet[Figure~5]{Kool2019Attention}, where routing performance degrades markedly off distribution. Related work has also studied the OOD behavior of transformer variants for algorithmic computation; notably, \citet{DeLuca2025NAR} report poor OOD performance on induction-head-style tasks. We observe the same qualitative effect in our own experiments (Sec.~\ref{sec:experiments}). A natural architectural reason is attention dilution: for fixed learned key-query geometry, attention becomes increasingly diffuse as the context length grows, so even a correctly learned heuristic may require mild recalibration when transported to larger $T$. At the same time, allowing polynomially many new samples would blur the distinction between correcting such architectural drift and relearning the task itself. Logarithmic adaptation provides a middle ground. In particular, \citet[Theorem 4.7]{Edelman2022} make a compelling theoretical case for $O(\log(T/T_0))$. 
% There, they derive log-covering number bounds for transformer blocks with logarithmic dependence on the context length, which implies uniform generalization bounds, supporting $O(\log(T/T_0))$ as a natural adaptation scale at which one may hope to restore uniform control of the generalization error beyond the original training horizon. 
Their results are consistent with our decision to formulate capture using a logarithmic re-adaptation budget rather than requiring strict zero-adaptation extrapolation.

\begin{definition}[EPTHS]
\label{def:epths}
An algorithm $A$ is an \emph{efficient polynomial-time heuristic scheme} (EPTHS) of complexity $O(T^k)$ for a combinatorial task $g$ if for every $\delta\in(0,1]$ there exists a $\delta$-implementation $A_\delta$ of $A$ that returns an output $A_\delta(X) \in \cY$ in time $O(\eta(1/\delta)T^k)$ such that 
\begin{align}
    \Pr\nolimits_{X \sim \mu_{X,T}}\left[\dist_{\cY}(A_\delta(X),g(X))<\Delta/2\right] > 1-\delta
\end{align}
for sufficiently large $T$.
\end{definition}
 
% Since distinct outputs in $\cY$ are separated by at least $\Delta$, each $\delta$-implementation induces a well-defined exact decision rule that is correct with probability at least $1-\delta$. 

In principle, the function $\eta$ in Definition~\ref{def:epths} is allowed to be exponential, or worse, in $1/\delta$. This is by analogy with the efficient polynomial-time approximation scheme (EPTAS), but with the guarantees reversed: whereas an EPTAS returns a $(1+\epsilon)$-approximate solution for every input, an EPTHS provides an exact recovery of the target output for most inputs, namely with probability above $1-\delta$ over $X\sim\mu_{X,T}$. Our use of the input measures $\mu_{X,T}$ and of heuristic complexity is inspired by the framework of distributional decision problems in average-case complexity theory, as developed by \citet{Bogdanov2006}. Since the tasks studied here have structured or scalar outputs rather than binary outputs, we adopt a broader task-based formulation while retaining the same distributional viewpoint.

\begin{remark}[From average-case to EPTHS]
\label{rem:average_case_to_epths}
A polynomial average-case complexity of a combinatorial task establishes an upper bound for EPTHS. Let $A$ be an algorithm that returns a correct outcome $g(X)$ for every $X \sim \mu_{X,T}$ in time $t(X)$ with $\bE_{X\sim \mu_{X,T}}[t(X)]=O(T^k)$. Then $A$ establishes an EPTHS of complexity $O(T^k)$ for $g$. By Markov's inequality
\begin{align}
  \Pr\nolimits_{X\sim \mu_{X,T}}\left[t(X) \geq \bE_{X\sim \mu_{X,T}}[t(X)]/\delta\right] \leq \delta.
\end{align}
Then, as a $\delta$-implementation $A_\delta$ of $A$, run the latter on each instance $X$ for up to $O(\bE_{X\sim \mu_{X,T}}[t(X)]/\delta)$ steps. If $A$ finished within the allotted step budget, provide its outcome, otherwise, output a random outcome for the task. By construction,
\begin{align}
    \Pr\nolimits_{X\sim \mu_{X,T}}\left[\dist_{\cY}(A_\delta(X), g(X))=0\right]>1-\delta
\end{align}
in time $O(T^k/\delta)$.
\end{remark}

%%%%%%%%%%%%%%%%%%%%%%%%%%%%%%%%
%%%%%%%%%%%%%%%%%%%%%%%%%%%%%%%%
%%%%%%%%%%%%%%%%%%%%%%%%%%%%%%%%

\subsection{Empirical Scaling Budgets on Various Tasks}
\label{sec:experiments}

Our empirical findings in Fig.~\ref{Fig:Grokking} illustrate a spectrum of capabilities that our framework helps to decompose. We observe instances of task capture for the Induction (Panel (a)) and Sorting Vocabulary (Panel (b))
% a list of integers between $1..V$ 
tasks.
% \footnote{For the Induction task, we do not allow repetition of the source-token, while for the Sorting Vocabulary and String Match tasks we do.} 
For the String Match task of a 3-token string (Panel (c)), we see capture on average, while individual seeds often, but not always, show several different slopes on our log-linear plots. In all these tasks EPTHS complexity is low and aligns well with finite circuit solutions and RASP programs \citep{weiss2021thinkingliketransformers}. 

We also observe instances suggesting lack of capture in the Context Free Grammar (CFG) task (Panel (g)), as predicted by our upper EPTHS complexity bound (Corollary~\ref{coro:epths_for_feed_forward}). On the other hand, we find combinatorial tasks, notably the source-target shortest path problem (SPP; Panel (e)) and the Max Flow / Minimal Cut (MinCut; Panel (f)) problem, whose native EPTHS complexity is allowed by our bounds. This does not refute our theory, as an upper EPTHS complexity bound is no guarantee for capture below it. In App.~\ref{app:combinatorial_tasks} we spell out each combinatorial task setting used in our experiments. Additional engineering choices are spelled out in App.~\ref{app:engineering}.

%The upper bounds we establish (Corollaries~\ref{coro:epths_for_kernel_evaluation} and \ref{coro:epths_for_feed_forward}), by and of themselves, are no guarantee for capturing combinatorial tasks of complexity below these bounds. Indeed, SPP and MinCut suggest failure of capture is possible below them. This failure should not be sought in the challenge of optimization per ce, since we were able to reduce test error in distribution as close to $0$ as desired, with the addition of data samples. We believe it should not be attributed to representational limitations either, since taking width to infinity eliminates expressivity concerns. More fundamentally, a possible root cause would be in the failure to ``internalize" an effective heuristic/algorithm, opting instead for easier optimization solutions that work well in distribution (i.e. a form of shortcut learning and simplicity bias).

\begin{figure*}[!htbp]
    \centering
    \begin{subfigure}[b]{\textwidth}
        \centering
        \includegraphics[trim={0cm 17.5cm 0 0},clip,width=\textwidth]{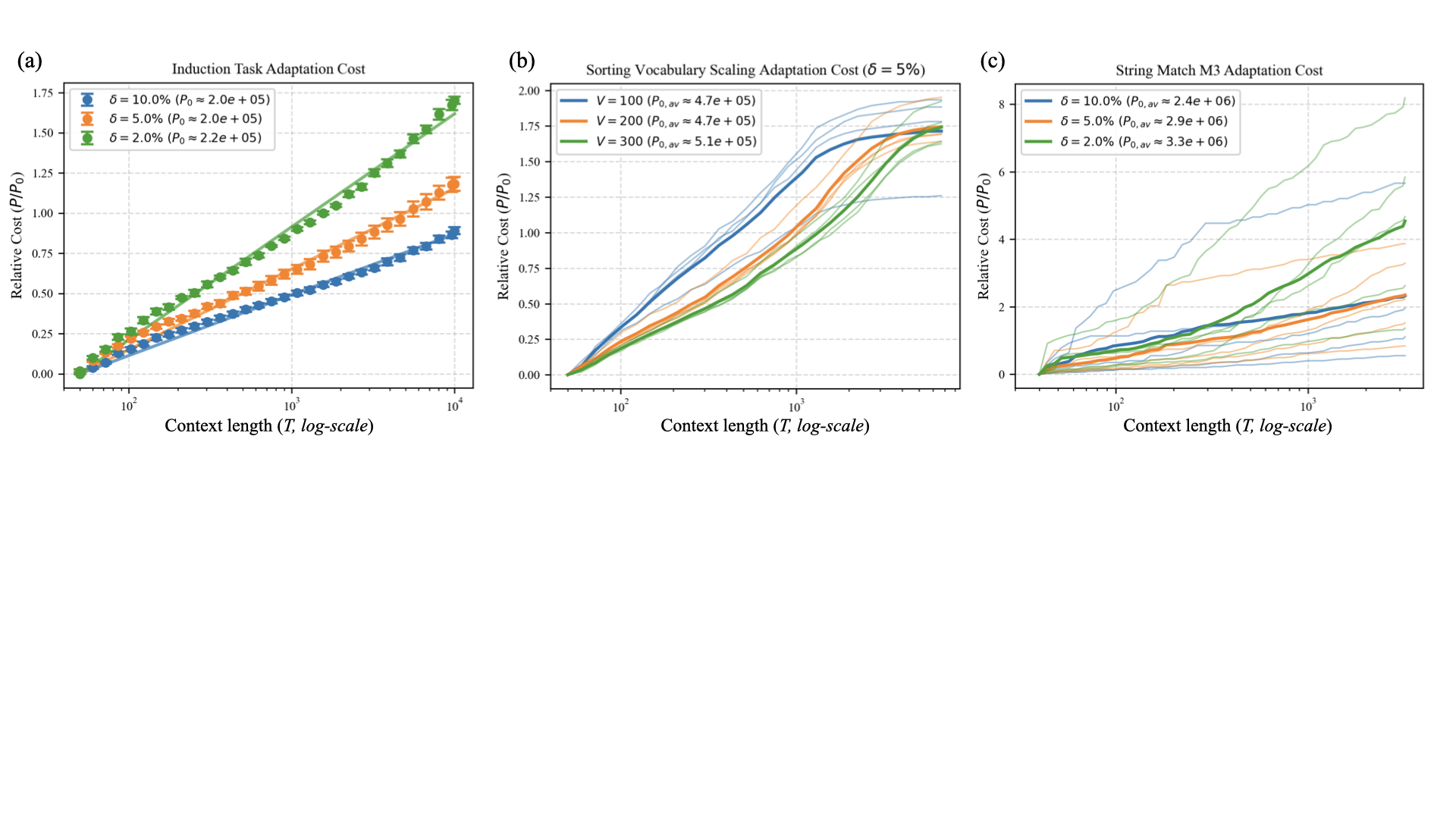}
        \label{fig:top_row}
    \end{subfigure}
    
    \vspace{-0.5cm} % Add some vertical spacing between rows

    % --- Bottom Row ---
     \begin{subfigure}[b]{\textwidth}
        \centering
        \includegraphics[trim={0cm 20cm 0 0},clip,width=\textwidth]{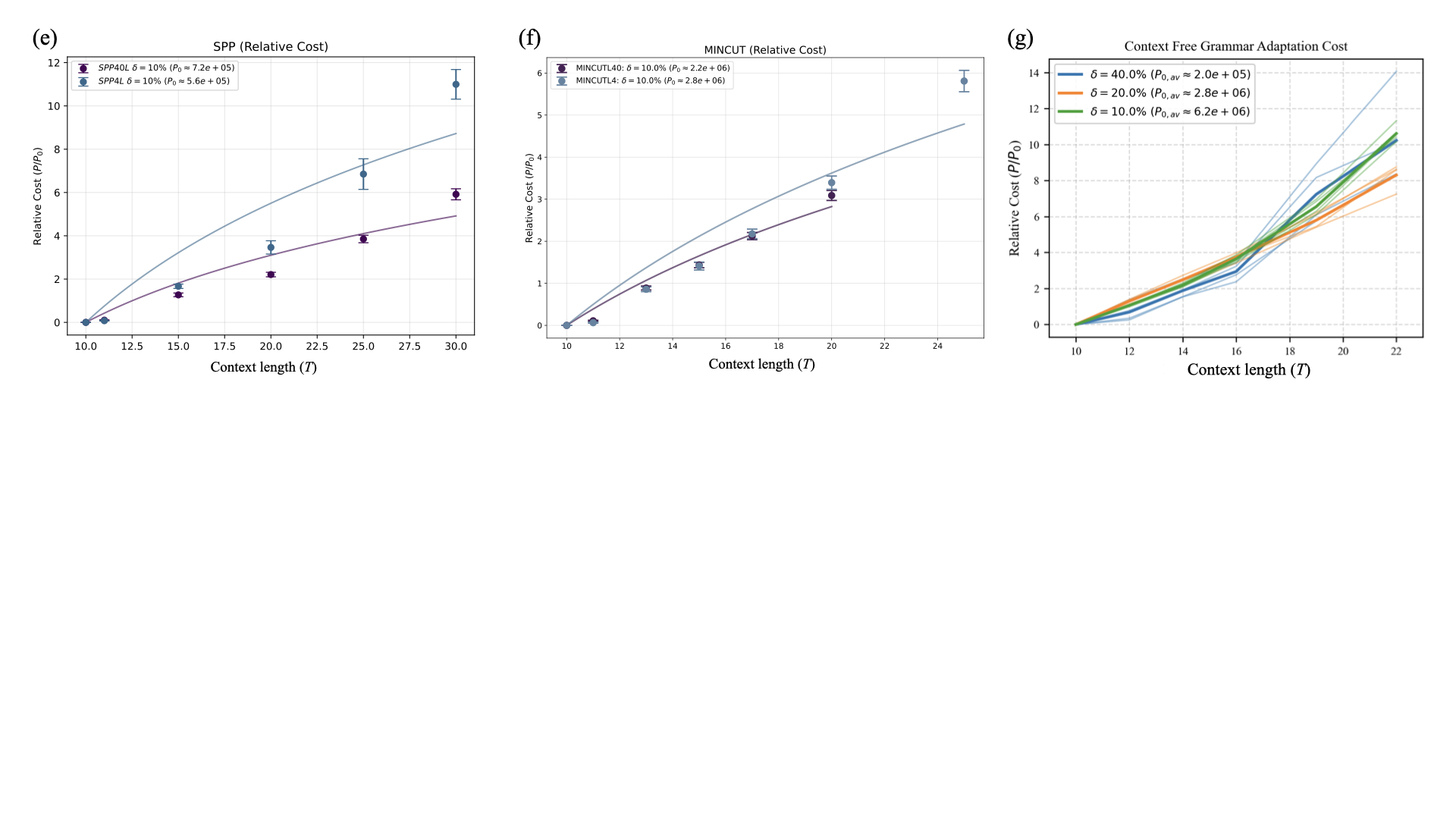}
        \label{fig:bottom_row}
    \end{subfigure}    
    \caption{\textbf{Empirical Combinatorial Task Capture.} 
    We train various models on combinatorial task instances of size $T_0$ (1st point on the x-axes) reaching accuracy $\delta$ after seeing $P_0(\delta)$ datapoints. We then fine-tune these models on larger instance sizes ($T>T_0$), re-achieving $\delta$-accuracy after $P$ extra datapoints. Panels (a)-(c) show linear-log plots of relative adaptation cost ($P/P_0$). Panels (a)-(b) show evidence of task capture consistent across seeds. Panel (c) shows task capture on average while individual seeds transition between different slopes. Panels (e)-(f) present results from a study of the combinatorial graph tasks, SPP and MinCut, that exhibit accelerating curves on a linear-linear scale. Solid lines are best fits to $C\log(T/T_0)$ meant to guide the eye. To mitigate expressivity concerns, depth $40$ networks were used for the purple datapoints in Panels (e)-(f). For the Context Free Grammar task on Panel (g), our theoretical bounds prohibit capture at fixed-depth regardless of model capacity. Panels (a,e,f) show average empirical curves with one STD error bars together with best fits to a log. Panels (b,c,g) report individual seeds (thin lines) as well as averages (thick lines).
\label{Fig:Grokking}}
\end{figure*}

%%%%%%%%%%%%%%%%%%%%%%%%%%%%%%%%
%%%%%%%%%%%%%%%%%%%%%%%%%%%%%%%%
%%%%%%%%%%%%%%%%%%%%%%%%%%%%%%%%

\section{Theoretical Complexity Inductive Bias}

Finite networks have inference-time complexity bounded by the number of FLOPs needed for their forward-pass, which, for transformers, scales as $O(LH\dmod T^2)$ with depth $L$, number of attention heads $H$, and embedding dimension $\dmod$ (see Lemma~\ref{lem:finite_width_complexity}). Thus, any task realized by a finite-width transformer must admit inference within this $O(T^2)$ budget, up to architecture-dependent constants. Infinitely-wide networks -- for transformers $H,d_{model}\rightarrow \infty$ -- are universal approximators \citep{Lee2018DNN} and, thus, expressively far richer. Still, their inductive biases, as we will show, tilt them towards low inference-time computational complexity. To this end, we seek to upper bound the inference-time complexity of a trained transformer in the extreme over-parameterized limit. 

%%%%%%%%%%%%%%%%%%%%%%%%%%%%%%%%
%%%%%%%%%%%%%%%%%%%%%%%%%%%%%%%%
%%%%%%%%%%%%%%%%%%%%%%%%%%%%%%%%

\subsection{Infinite-width Predictor}

We now turn to infinite transformer networks in the sense of \citet{Hron2020}. In this limit, the network predictor $f(X)$ converges to a kernel predictor. At initialization, or in the context of Bayesian neural networks and SGLD training \citep{NavehSompolinskyRingel,Welling2011}, the relevant kernel is the neural network Gaussian process (NNGP) kernel \citep{Lee2018DNN}. For gradient flow training (i.e., gradient descent at vanishing learning rate), the relevant kernel is the neural tangent kernel (NTK) \citep{Jacot2018,Hu2020,Hron2020}. In either case, the predictor takes the form of a weighted kernel expansion over the training dataset $\{X^{(\nu)},y^{(\nu)}\}_{\nu=1}^P$,
\begin{align}
\label{eq:kernel_predictor}
    % f_{K}(X) = \bm{\alpha}^\top K(X,\bm{X}) ,\quad\bm{\alpha} = (K(\bm{X},\bm{X})+\kappa I_P)^{-1}\bm{y},
    f_{K}(X) = \sum_{\nu=1}^P \alpha_\nu K(X,X^{(\nu)}),\quad
    \bm{\alpha} = (K+\kappa I_P)^{-1}\bm{y},
\end{align}
where $K_{\mu\nu}=K(X^{(\mu)},X^{(\nu)})$ is the Gram matrix of dimension $P \times P$ and $\bm{y}$ is the response vector.
Thus, rather than performing an infinite-compute forward-pass through the network, we may analyze the computational complexity of evaluating Eq.~\eqref{eq:kernel_predictor}, namely, the cost of computing the $P$ entries $K(X,X^{(\nu)})$. As we are evaluating inference-time complexity, the matrix inversion in $\bm{\alpha}$ is not incorporated; we consider it part of training complexity. Accordingly, we analyze the number of FLOPs required to compute $K(\cdot,\cdot)$.

\begin{remark}[Kernel two-step training protocol]
The two-step training protocol of Definition~\ref{def:capture}, which is the one used in our experiments, admits a corresponding kernel formulation \citep{bennani2020generalisation}. In this formulation, the first-step predictor, $f_K^{(1)}(X)$, is the kernel predictor of Eq.~\eqref{eq:kernel_predictor} with $P_0$ initial data samples. The second-step predictor, $f_K^{(2)}(X)$, is obtained by adding to $f_K^{(1)}(X)$ the kernel predictor on the second-stage $P_1$ labeled data samples, $\{\tilde{X}^{(\nu)},\tilde{y}^{(\nu)}\}_{\nu=1}^{P_1}$, fitted to the residuals of the first stage,
\begin{align}
\label{eq:two_step_kernel}
    % f_K^{(2)}(X) = f_K^{(1)}(X) + K(X,\bm{X}^{(2)})\left(K(\bm{X}^{(2)},\bm{X}^{(2)})+\kappa I_P\right)^{-1}\left(\bm{y}^{(2)}-f_K^{(1)}(\bm{X}^{(2)})\right).
    % f_K^{(2)}(X) = f_K^{(1)}(X) + \sum_{\mu,\nu=1}^P K(X,\tilde{X}^{(\mu)})\left(\tilde{K}+\kappa I_P\right)^{-1}_{\mu\nu}\left(\tilde{y}^{(\nu)}-f_K^{(1)}(\tilde{X}^{(\nu)})\right)
    f_K^{(2)}(X) = f_K^{(1)}(X) + \sum_{\nu=1}^{P_1} \tilde{\alpha}_\nu K(X,\tilde{X}^{(\nu)}), \  \tilde{\alpha}_\nu = \!\sum_{\mu=1}^{P_1} \left(\tilde{K}+\kappa I_{P_1}\right)^{-1}_{\nu \mu}\!\left(\tilde{y}^{(\mu)}-f_K^{(1)}(\tilde{X}^{(\mu)})\right),
\end{align}
for the second-stage Garm matrix $\tilde{K}_{\mu\nu}=K(\tilde{X}^{(\mu)},\tilde{X}^{(\nu)})$. Here, as well, both $\bm \alpha$ and $\tilde{\bm \alpha}$ are incorporated into training-time complexity.
\end{remark}

\begin{remark}[Bounded kernel coefficients]
\label{rem:bounded_kernel_weights}
We assume that the kernel coefficients $\abs{\alpha_{\nu}}$ are uniformly bounded by a constant independent of $P$, as is typically expected. Our results in Sec.~\ref{sec:kernel_evaluation_complexity}, and Proposition~\ref{prop:BlockPropagationComplexity} in particular, also accommodate the weaker bound $\abs{\alpha_{\nu}}=O(\sqrt{P})$, which corresponds to the worst case for finite-ridge regression. This is handled by scaling $N_{MC}$, defined in Proposition~\ref{prop:BlockPropagationComplexity}, by an additional factor of ${P}$.
\end{remark}

% If $\abs{\alpha_{\nu}}=O(\sqrt{P})$ then
% \begin{align}
%     \mathrm{Var}\left(\sum_{\nu=1}^P \alpha_\nu K(X,X^{(\nu)})\right) = \sum_{\nu=1}^P \mathrm{Var}\left(\alpha_\nu\right)\mathrm{Var}\left(K(X,X^{(\nu)})\right) = P \cdot P \cdot N_{\mtc}
% \end{align}
% and $N_{MC}$ gets an additional factor of $P$.

%%%%%%%%%%%%%%%%%%%%%%%%%%%%%%%%
%%%%%%%%%%%%%%%%%%%%%%%%%%%%%%%%
%%%%%%%%%%%%%%%%%%%%%%%%%%%%%%%%

\subsection{Covariance Propagation}
\label{sec:kernel_propagation}

In the NNGP setting, the sufficient statistic describing the transformer block state is the token-to-token covariance matrix $\Cov'(X_1,X_2) \in \bR^{T \times T}$, for fixed inputs $X_1,X_2 \in \bR^{T \times d}$, with entries
\begin{align}
\label{eq:covariance}
    \Cov'_{ab}(X_1,X_2) = \lim_{\dmod \to \infty} \frac{1}{\dmod} \bE \left[  \bm{z'}_{a}(X_1)^\top \bm{z'}_{b}(X_2) \right],
\end{align}
where $\bm{z'}_{a}(X)$ are the outputs of the MLP layer of the previous block (App.~\ref{app:transformer_setup}). In the infinite-width limit, $(\bm{z'}_{a}(X))_{a\in[T]}$ is a centered multivariate Gaussian. Evaluating the infinite-width NNGP kernel then amounts to propagating this token-to-token covariance matrix, in the sense of \citet{ChoSaul2009}, through the transformer blocks, starting from its attention layer. The final-block value then determines $K(X_1,X_2)$. In the infinite-width limit, we write $S(X) \in \bR^{T \times T}$ for the limiting attention score matrix and $A(X) = \Softmax(S(X))$ for the corresponding limiting attention matrix. 

\begin{restatable}[Attention layer propagation, \HronAttentionCitationII]{proposition}{AttentionCovarianceUpdate}
\label{prop:AttentionCovarianceUpdate}
Assume the incoming covariance matrix $\Cov'(X_1,X_2)$ is given. In the infinite-width limit, the token-to-token covariance matrix propagated through the attention layer is
\begin{align}
\label{eq:infinite_covariance_update}
\Cov(X_1,X_2) = \bE_{(S(X_1),S(X_2))}\!\left[A(X_1)\Cov'(X_1,X_2)A(X_2)^\top \right].
\end{align}
\end{restatable}

The proof of Proposition~\ref{prop:AttentionCovarianceUpdate} uses the law of the attention score matrix $S(X)$ in the infinite-width limit as given in following proposition. 

\begin{restatable}[Attention score distribution, \HronAttentionCitationI]{proposition}{CovarianceFactorization}
\label{prop:CovarianceFactorization}
    In the infinite-width limit, the attention score matrix $S(X)$ is a centered multivariate Gaussian, and its covariance is
    \begin{align}
    \label{eq:Factorization}
        \bE\left[S_{ac}(X_1) S_{be}(X_2)\right] = \Cov'_{ab}(X_1,X_2) \Cov'_{ce}(X_1,X_2).
    \end{align} 
\end{restatable}

Crucially, we will use the Kronecker tensor structure of the attention score covariance, made explicit in Proposition~\ref{prop:CovarianceFactorization}, to obtain an MC estimator of the right-hand side of Eq.~\eqref{eq:infinite_covariance_update} with efficient score sampling. We provide proofs of Propositions~\ref{prop:AttentionCovarianceUpdate} and \ref{prop:CovarianceFactorization}, simplified to our setting, in App.~\ref{app:attention_layer_propagation} and App.~\ref{app:attention_scores}, respectively.

%%%%%%%%%%%%%%%%%%%%%%%%%%%%%%%%
%%%%%%%%%%%%%%%%%%%%%%%%%%%%%%%%
%%%%%%%%%%%%%%%%%%%%%%%%%%%%%%%%

\subsection{Kernel Evaluation Complexity}
\label{sec:kernel_evaluation_complexity}

To analyze the computational complexity of the kernel predictor, we begin with the cost of propagating the token-to-token covariance matrix and the NTK through a transformer block, considering the cost of kernel evaluation using MC integration.
Throughout this section, we fix a transformer block in the NNGP or NTK limit and assume the incoming token-to-token covariance matrix $\Cov'(X_1,X_2) \in \bR^{T \times T}$ of Eq.~\eqref{eq:covariance} or the incoming NTK
\begin{align}
\label{eq:ntk}
    \Theta'_{ab}(X_1,X_2) =
    \lim_{\dmod\to\infty}
    \frac{1}{\dmod}
    \mathbb E\!\left[
        \nabla_{\theta} \bm z'_a(X_1)^\top
        \nabla_{\theta} \bm z'_b(X_2)
    \right],\quad a,b\in[T]
\end{align}
are given. We assume the coefficients $\alpha_\nu$ of the kernel predictor induced by $\Cov(X_1,X_2)$ and $\Theta(X_1,X_2)$ satisfy the bounded-weight assumption of Remark~\ref{rem:bounded_kernel_weights}. In our setting, we also assume the activation function $\phi$ is bounded by linear growth (e.g., ReLU or GeLU).

\begin{restatable}[Transformer block propagation]{proposition}{BlockPropagationComplexity}
\label{prop:BlockPropagationComplexity}
The computational complexity of evaluating the MC estimator of the predictor associated with the NNGP / NTK kernel propagated through a single transformer block is $O(P N_{\mtc}T^3)$, with MC error scaling as $O(\sqrt{P/N_{\mtc}})$, where $N_{\mtc}$ denotes the number of MC samples. \emph{(Proof in App.~\ref{app:transformer_block_propagation}.)}
\end{restatable}

In App.~\ref{app:transformer_propagation} we argue that the scaling of error and compute of Proposition~\ref{prop:BlockPropagationComplexity} extend to the full transformer. The key observation is that, while MC error can accumulate across blocks with the noisy versions of the true input block-wise $\Cov'$ or $\Theta'$, this accumulation is $T$-independent since the kernel recursion relations of the attention and $\mlp$ layers admit $T$-independent Lipschitz constants. Thus, to maintain accuracy, one only need scale the number of MC samples with depth $L$, not with context length $T$. 

\begin{restatable}[Kernel evaluation complexity]{theorem}{TransformerPropagationComplexity}
\label{thm:TransformerPropagationComplexity}
The computational complexity of evaluating the MC estimator of the NNGP / NTK kernel predictor is $O(PN_{\mtc}T^3)$, with MC error scaling as $O(\sqrt{P/ N_{\mtc}})$. \emph{(Proof in App.~\ref{app:transformer_propagation}.)}
\end{restatable}

Note that the computational complexity of evaluating the MC estimator will depend linearly on depth $L$, while the MC error will depend exponentially on $L$. For a finite-depth network, these are fixed constant factors, agnostic to $T$.

\begin{corollary}[Kernel evaluation EPTHS]
\label{coro:epths_for_kernel_evaluation}
    If an infinite-width transformer has algorithmically captured a combinatorial task $g$, then it is an EPTHS for $g$ of complexity $O(T^{3+\epsilon})$ for any $\epsilon>0$.
\end{corollary}

\begin{proof}
Assume now that a transformer in the infinite-width limit has captured a combinatorial task $g$ under $\mu_{X,T}$. Then, for sufficiently small $\delta>0$, there exist $T_0=T_0(\delta/2)$, $C=C(\delta/2)$, and $P_0 = P_0(\delta/2)$ such that the kernel predictor $f_\delta$ of Eq.~\eqref{eq:two_step_kernel} with $P_1=C\log(T/T_0)$ adaptation data samples satisfies
\begin{align}
    \Pr\nolimits_{X \sim \mu_{X,T}} \left[ \dist_{\cY}(f_\delta(X),g(X)) < \Delta/3\right] > 1-\delta/2, \quad \forall T \geq T_0.
\end{align}
Using MC integration to evaluate $f_\delta(X)$, we get $f_\delta^{\mtc}(X)$ with the MC error scaling of Theorem~\ref{thm:TransformerPropagationComplexity} for $P=P_0+P_1$. By the sub-Gaussian tail bound, we have 
\begin{align}
    \Pr\nolimits_{X \sim \mu_{X,T}}\left[\dist_{\cY}(f_\delta^{\mtc}(X),f_\delta(X)) < \Delta/6\right] \geq 1-2e^{-c\Delta^2 N_{\mtc}/P}.
\end{align}
Applying the triangle inequality, we have
\begin{align}
    \Pr\nolimits_{X \sim \mu_{X,T}} \left[ \dist_{\cY}(f_\delta^{\mtc}(X),g(X)) < \Delta/2\right] > 1- 2e^{-c\Delta^2 N_{\mtc}/P} -\delta/2,
\end{align}
and we want $2e^{-c\Delta^2 N_{\mtc}/P} \approx \delta/2$, which imposes $ N_{\mtc} \propto P\log(1/\delta)$. Evaluating $f_\delta^{\mtc}(X)$, by Theorem~\ref{thm:TransformerPropagationComplexity}, requires $O(P N_{\mtc}{T^3})$ computational steps. Thus, the infinite-width transformer is an EPTHS for the task it algorithmically captured with complexity $O(T^{3+\epsilon})$ w.r.t. $\mu_{X,T}$, for any $\epsilon>0$, coming from $\log(T)^2$. Its $\delta$-implementation is given by $f_\delta^{\mtc}(X)$, computed in time $O(P^2\log(1/\delta){T^3})$. Notably, $\Delta/2$ provides a conclusive decision rule.
\end{proof}

\begin{remark}[Polynomial data adaptation]
If we were to adapt training data polynomially, say with $O(T^\kappa)$ extra samples, we could push this adaptation protocol through our MC-based kernel analysis to get an $O(T^{3+2\kappa})$ complexity bound for EPTHS for the infinite-width transformer at initialization. Our viewpoint, reflected here, is that the most reasonable separator of ``understanding'' from statistical interpolation is at $\kappa \rightarrow 0$. 
% At the very least, this threshold leaves out a non-empty class of non-capturable tasks.
\end{remark}

%%%%%%%%%%%%%%%%%%%%%%%%%%%%%%%%
%%%%%%%%%%%%%%%%%%%%%%%%%%%%%%%%
%%%%%%%%%%%%%%%%%%%%%%%%%%%%%%%%

\subsection{Feed-forward Approximation Complexity}
\label{sec:tighter_bounds}

The above results invite extensions in terms of tightness and richness. Unlike in FCNs (App.~\ref{app:FCN}), kernel evaluation for transformers is costlier than a forward-pass, because it requires tracking correlations over the entire context. By contrast, various bounds on NTK convergence \citep{huang2020dynamics,Jacot2018,huang2021neuraltangentkerneldeep} suggest that, at least for MLPs and CNNs, the discrepancy between the NTK kernel predictor and the finite-width network scales as $P^{\gamma}/N$, for some architecture-dependent constant $\gamma$, where $N$ is the architecture width or redundancy scale: width for MLPs, channels for CNNs, or $H$ and $\dmod$ for transformers. 

\begin{assumption}[NTK convergece]
\label{assm:discrepancy_scale}
    The discrepancy between the NTK kernel predictor and the finite transformer network scales as $P^{\gamma}/N$, where $N$ is the architecture redundancy scale and $\gamma$ is an architecture-dependent constant. 
\end{assumption}

To the best of our knowledge, how well this scaling holds in extreme OOD settings such as ours has not been fully explored yet. Finite-width approximation of the rich infinite-width NTK predictor has been studied in the literature \citep{bordelon2023dynamicsfinitewidthkernel,Seroussi2021}, potentially with $\gamma \approx 0$. We provide theoretical evidence that Assumption~\ref{assm:discrepancy_scale} nevertheless holds in the following proposition.

\begin{restatable}[NNGP convergence bound]{proposition}{NNGPConvergenceBound}
\label{prop:NNGPConvergenceBound}
    Consider a finite-width transformer of width $N$, trained with full-batch gradient descent (GD), weight decay, and additive white Gaussian noise. Let $f(X)=\bm z_T^{\mathrm{out}}(X)$ for $X \in \bR^{T \times d}$, and let $f_K(X)$ be its NNGP kernel predictor with $K = \Cov_{TT}^{\mathrm{out}}$. Then, in standard scaling,
    \begin{align}
        \abs{\bE_{f \sim P[f]}[{f}(X_*)] - f_{K}(X_*)} = O\left(\frac{ P^{\gamma}}{N}\right)
    \end{align}
    where $P[f]$ \citep[Eq.~(3)]{NavehSompolinskyRingel} is the posterior on the trained $f$. (\emph{Proof in App.~\ref{app:perturbation_theory}.})
\end{restatable}

\begin{corollary}[Feed-forward EPTHS]
\label{coro:epths_for_feed_forward}
    Under Assumption~\ref{assm:discrepancy_scale}, if an infinite-width transformer at the lazy or rich regime has algorithmically captured a combinatorial task $g$, then it is an EPTHS for $g$ of complexity $O(T^{2+\epsilon})$ for any $\epsilon>0$.
\end{corollary}

\begin{proof}
Under Assumption~\ref{assm:discrepancy_scale}, if an infinite-width transformer has captured a combinatorial task, then $N$ need only scale poly-logarithmically with $T$. As a result, we can approximate the infinite-width NTK kernel predictor using a forward-pass on a finite trained network with $\dmod,H \propto P=\log(T)^{\gamma}$, getting computational complexity scaling of $O(\log(T)^{2\gamma}T^2)$ by Lemma~\ref{lem:finite_width_complexity}. We thus get computational complexity scaling of $T^2$ times negligible poly-logarithmic corrections.
\end{proof}

% Corollary~\ref{coro:epths_for_feed_forward} is a tighter EPTHS complexity upper bound compared to the $O(T^{3+\epsilon})$ of Corollary~\ref{coro:epths_for_kernel_evaluation}. We note that Corollary~\ref{coro:epths_for_feed_forward} can be further improved down to $O(T^{1+\epsilon})$ for rank-limited attention. We further note that the emergence of learned circuits in the rich regime, whose number may depend polynomially on $P$, again brings us to $T$-dependence scaling, as in the forward-pass computation on a finite network.

%%%%%%%%%%%%%%%%%%%%%%%%%%%%%%%%
%%%%%%%%%%%%%%%%%%%%%%%%%%%%%%%%
%%%%%%%%%%%%%%%%%%%%%%%%%%%%%%%%

\section{Discussion}

In this work, we moved beyond the question of what transformers can express -- which, in the infinite-width limit, is effectively everything -- to ask what combinatorial tasks they can capture. By formally defining algorithmic capture (Definition~\ref{def:capture}) as the ability to generalize to arbitrary input sizes ($T$) with logarithmic sample adaptation, we offered a distinction between genuine logic internalization and statistical interpolation. Our theoretical results for infinite-width transformers established an upper bound on their inference-time complexity, limiting them to the class of EPTHS (Definition~\ref{def:epths}) with complexity scaling no worse than $O(T^{3+\epsilon})$ for $\epsilon>0$ (Corollary~\ref{coro:epths_for_kernel_evaluation}) and $O(T^{2+\epsilon})$ (Corollary~\ref{coro:epths_for_feed_forward}) under Assumption~\ref{assm:discrepancy_scale}.  

The case of SPP and MinCut (Sec.~\ref{sec:experiments}) highlights a fruitful avenue of future research to refine the estimate of inference-time complexity. Our Corollary~\ref{coro:epths_for_kernel_evaluation} bound relies on a ``brute force'' propagation of kernel covariances, while our Corollary~\ref{coro:epths_for_feed_forward} assumes NTK convergence (Assumption~\ref{assm:discrepancy_scale}). While robust, this approach may be overly conservative. In fact, we conjecture that the dominant spectral weight of the kernel lies in the space of rather simple sparse functions.

Finally, it would be interesting to generalize these results to deeper or architectures or chain of thought (CoT); the fact that these deeper nets can express complex tasks does not mean that they would learn them, as various works and our own experiments suggest. At least for lazy learning, where our bound is largely determined by the complexity of kernel evaluation, this amounts to understanding fixed points of the kernel recursion equations, which are simple for MLPs \citep{Schoenholz2016} but remain unexplored for transformers. Turning to CoT, say in the context of reinforcement learning with verifiable rewards (RLVR) \citep{Wen2026RLVR}, our definition of capture applies verbatim, as it does not rely on the underlying inference mechanism or loss. Establishing meaningful bounds on the complexity of learnable functions using RLVR, is left for future work.

%%%%%%%%%%%%%%%%%%%%%%%%%%%%%%%%
% BIBLIOGRAPHY
%%%%%%%%%%%%%%%%%%%%%%%%%%%%%%%%
\bibliographystyle{plainnat}
\bibliography{ref}

@inproceedings{Barak2022,
  title     = {Hidden Progress in Deep Learning: SGD Learns Parities Near the Computational Limit},
  author    = {Barak, Boaz and Edelman, Benjamin L and Goel, Surbhi and Kakade, Sham and Malach, Eran and Zhang, Cyril},
  booktitle = {Advances in Neural Information Processing Systems},
  volume    = {35},
  year      = {2022},
  url       = {https://proceedings.neurips.cc/paper_files/paper/2022/file/884baf65392170763b27c914087bde01-Paper-Conference.pdf}
}

@inproceedings{Nanda2023,
  title     = {Progress measures for grokking via mechanistic interpretability},
  author    = {Nanda, Neel and Chan, Lawrence and Lieberum, Tom and Smith, Jess and Steinhardt, Jacob},
  booktitle = {International Conference on Learning Representations},
  year      = {2023},
  url       = {https://openreview.net/pdf?id=9XFSbDPmdW}
}

@book{vonMises1964Probability,
  author    = {von Mises, Richard},
  title     = {Mathematical Theory of Probability and Statistics},
  publisher = {Academic Press},
  address   = {New York and London},
  year      = {1964}
}

@article{NavehSompolinskyRingel,
  title = {Predicting the outputs of finite deep neural networks trained with noisy gradients},
  author = {Naveh, Gadi and Ben David, Oded and Sompolinsky, Haim and Ringel, Zohar},
  journal = {Phys. Rev. E},
  volume = {104},
  issue = {6},
  pages = {064301},
  numpages = {19},
  year = {2021},
  month = {Dec},
  publisher = {American Physical Society},
  doi = {10.1103/PhysRevE.104.064301},
  url = {https://link.aps.org/doi/10.1103/PhysRevE.104.064301}
}

@article{Ghorbani2021,
  title   = {When do neural networks outperform kernel methods?},
  author  = {Ghorbani, Behrooz and Mei, Song and Misiakiewicz, Theodor and Montanari, Andrea},
  journal = {Advances in Neural Information Processing Systems},
  volume  = {33},
  pages   = {14820--14830},
  year    = {2020},
  url     = {https://proceedings.neurips.cc/paper/2020/file/a9df2255ad642b923d95503b9a7958d8-Paper.pdf}
}

@inproceedings{Wei2019,
  title   = {Regularization matters: Generalization and optimization of neural nets v.s. their induced kernel},
  author  = {Wei, Colin and Lee, Jason D and Liu, Qiang and Ma, Tengyu},
  booktitle = {Advances in Neural Information Processing Systems},
  volume  = {32},
  year    = {2019},
  url     = {https://arxiv.org/abs/1810.05369}
}

@article{Bietti2022,
  title   = {Learning Single-Index Models with Shallow Neural Networks},
  author  = {Bietti, Alberto and Bruna, Joan and Sanford, Clayton and Song, Min Jae},
  journal = {Advances in Neural Information Processing Systems},
  volume  = {35},
  pages   = {9768--9780},
  year    = {2022},
  url     = {https://proceedings.neurips.cc/paper_files/paper/2022/file/3fb6c52aeb11e09053c16eabee74dd7b-Paper-Conference.pdf}
}

@book{Bertsekas1998Network,
  title     = {Network Optimization: Continuous and Discrete Models},
  author    = {Dimitri P. Bertsekas},
  publisher = {Athena Scientific},
  address   = {Belmont, MA},
  year      = {1998},
  isbn      = {1-886529-02-7},
  url       = {http://www.athenasc.com/netopt.html}
}

@article{Bogdanov2006,
  title     = {Average-Case Complexity},
  author    = {Andrej Bogdanov and Luca Trevisan},
  journal   = {Foundations and Trends in Theoretical Computer Science},
  volume    = {2},
  number    = {1},
  pages     = {1--106},
  year      = {2006},
  doi       = {10.1561/0400000004}
}

@inproceedings{Hu2020,
  author       = {Wei Hu and
                  Zhiyuan Li and
                  Dingli Yu},
  title        = {Simple and Effective Regularization Methods for Training on Noisily
                  Labeled Data with Generalization Guarantee},
  booktitle    = {8th International Conference on Learning Representations, {ICLR} 2020,
                  Addis Ababa, Ethiopia, April 26-30, 2020},
  publisher    = {OpenReview.net},
  year         = {2020},
  url          = {https://openreview.net/forum?id=Hke3gyHYwH},
  bibsource    = {dblp computer science bibliography, https://dblp.org}
}

@inproceedings{Welling2011,
author = {Welling, Max and Teh, Yee Whye},
title = {Bayesian learning via stochastic gradient langevin dynamics},
year = {2011},
isbn = {9781450306195},
publisher = {Omnipress},
address = {Madison, WI, USA},
booktitle = {Proceedings of the 28th International Conference on International Conference on Machine Learning},
pages = {681–688},
numpages = {8},
location = {Bellevue, Washington, USA},
series = {ICML'11}
}

@inproceedings{
zhou2024what,
title={What Algorithms can Transformers Learn? A Study in Length Generalization},
author={Hattie Zhou and Arwen Bradley and Etai Littwin and Noam Razin and Omid Saremi and Joshua M. Susskind and Samy Bengio and Preetum Nakkiran},
booktitle={The Twelfth International Conference on Learning Representations},
year={2024},
url={https://openreview.net/forum?id=AssIuHnmHX}
}

@article{Seroussi2021,
	Author = {Seroussi, Inbar and Naveh, Gadi and Ringel, Zohar},
	Da = {2023/02/17},
	Doi = {10.1038/s41467-023-36361-y},
	Id = {Seroussi2023},
	Isbn = {2041-1723},
	Journal = {Nature Communications},
	Number = {1},
	Pages = {908},
	Title = {Separation of scales and a thermodynamic description of feature learning in some CNNs},
	Ty = {JOUR},
	Url = {https://doi.org/10.1038/s41467-023-36361-y},
	Volume = {14},
	Year = {2023}}

@misc{bordelon2023dynamicsfinitewidthkernel,
      title={Dynamics of Finite Width Kernel and Prediction Fluctuations in Mean Field Neural Networks}, 
      author={Blake Bordelon and Cengiz Pehlevan},
      year={2023},
      eprint={2304.03408},
      archivePrefix={arXiv},
      primaryClass={stat.ML},
      url={https://arxiv.org/abs/2304.03408}, 
}

@article{Chi1999,
author = {Chi, Zhiyi},
title = {Statistical properties of probabilistic context-free grammars},
year = {1999},
issue_date = {March 1999},
publisher = {MIT Press},
address = {Cambridge, MA, USA},
volume = {25},
number = {1},
issn = {0891-2017},
journal = {Comput. Linguist.},
month = mar,
pages = {131–160},
numpages = {30}
}

@inproceedings{Lavie2024,
author = {Lavie, Itay and Gur-Ari, Guy and Ringel, Zohar},
title = {Towards understanding inductive bias in transformers: a view from infinity},
year = {2024},
publisher = {JMLR.org},
booktitle = {Proceedings of the 41st International Conference on Machine Learning},
articleno = {1042},
numpages = {27},
location = {Vienna, Austria},
series = {ICML'24}
}

@misc{huang2021neuraltangentkerneldeep,
      title={On the Neural Tangent Kernel of Deep Networks with Orthogonal Initialization}, 
      author={Wei Huang and Weitao Du and Richard Yi Da Xu},
      year={2021},
      eprint={2004.05867},
      archivePrefix={arXiv},
      primaryClass={cs.LG},
      url={https://arxiv.org/abs/2004.05867}, 
}

@inproceedings{huang2020dynamics,
  title={Dynamics of deep neural networks and neural tangent hierarchy},
  author={Huang, Jiaoyang and Yau, Horng-Tzer},
  booktitle={International conference on machine learning},
  pages={4542--4551},
  year={2020},
  organization={PMLR}
}

@article{abbe2024generalization,
  title={Generalization on the unseen, logic reasoning and degree curriculum},
  author={Abbe, Emmanuel and Bengio, Samy and Lotfi, Aryo and Rizk, Kevin},
  journal={Journal of Machine Learning Research},
  volume={25},
  number={331},
  pages={1--58},
  year={2024}
}

@article{bennani2020generalisation,
  title={Generalisation guarantees for continual learning with orthogonal gradient descent},
  author={Bennani, Mehdi Abbana and Doan, Thang and Sugiyama, Masashi},
  journal={arXiv preprint arXiv:2006.11942},
  year={2020}
}

@misc{weiss2021thinkingliketransformers,
      title={Thinking Like Transformers}, 
      author={Gail Weiss and Yoav Goldberg and Eran Yahav},
      year={2021},
      eprint={2106.06981},
      archivePrefix={arXiv},
      primaryClass={cs.LG},
      url={https://arxiv.org/abs/2106.06981}, 
}

@article{elhage2021mathematical,
  title   = {A Mathematical Framework for Transformer Circuits},
  author  = {Elhage, Nelson and Nanda, Neel and Olsson, Catherine and Henighan, Tom and Joseph, Nicholas and Mann, Ben and Askell, Amanda and Bai, Yuntao and Chen, Anna and Conerly, Tom and DasSarma, Nova and Drain, Dawn and Ganguli, Dip and Hatfield-Dodds, Zac and Hernandez, Danny and Jones, Andy and Kernion, Jackson and Lovitt, Liane and Ndousse, Kamal and Amodei, Dario and Brown, Tom and Clark, Jack and Kaplan, Jared and McCandlish, Sam and Olah, Chris},
  journal = {Transformer Circuits Thread},
  year    = {2021},
  url     = {https://transformer-circuits.pub/2021/framework/index.html}
}

@article{Dingle2018,
	Author = {Dingle, Kamaludin and Camargo, Chico Q. and Louis, Ard A.},
	Da = {2018/02/22},
	Doi = {10.1038/s41467-018-03101-6},
	Id = {Dingle2018},
	Isbn = {2041-1723},
	Journal = {Nature Communications},
	Number = {1},
	Pages = {761},
	Title = {Input--output maps are strongly biased towards simple outputs},
	Ty = {JOUR},
	Url = {https://doi.org/10.1038/s41467-018-03101-6},
	Volume = {9},
	Year = {2018}}

@misc{VallePerez2018,
      title={Deep learning generalizes because the parameter-function map is biased towards simple functions}, 
      author={Guillermo Valle-Pérez and Chico Q. Camargo and Ard A. Louis},
      year={2019},
      eprint={1805.08522},
      archivePrefix={arXiv},
      primaryClass={stat.ML},
      url={https://arxiv.org/abs/1805.08522}, 
}

@inproceedings{Chiang2023Tighter,
  title     = {Tighter Bounds on the Expressivity of Transformer Encoders},
  author    = {David Chiang and Peter Cholak and Anand Pillay},
  booktitle = {Proceedings of the 40th International Conference on Machine Learning (ICML)},
  volume    = {202},
  pages     = {5544--5562},
  year      = {2023},
  publisher = {PMLR},
  url       = {https://proceedings.mlr.press/v202/chiang23a.html}
}

@article{Chen2024Theoretical,
  title     = {Theoretical Limitations of Multi-Layer Transformer},
  author    = {Lijie Chen and Binghui Peng and Hongxun Wu},
  journal   = {arXiv preprint arXiv:2412.02975},
  year      = {2024},
  url       = {https://arxiv.org/abs/2412.02975}
}

@inproceedings{ChoSaul2009,
  title     = {Kernel Methods for Deep Learning},
  author    = {Youngmin Cho and Lawrence Saul},
  booktitle = {Advances in Neural Information Processing Systems (NIPS)},
  volume    = {22},
  pages     = {342--350},
  year      = {2009},
  url       = {https://proceedings.neurips.cc/paper/2009/hash/5751ec3e9a4feab575962e78e006250d-Abstract.html}
}

@article{Dall2002,
  title     = {Random Geometric Graphs},
  author    = {Jesper Dall and Michael Christensen},
  journal   = {Physical Review E},
  volume    = {66},
  number    = {1},
  pages     = {016121},
  year      = {2002},
  doi       = {10.1103/PhysRevE.66.016121}
}

@inproceedings{Dong2021Attention,
  title     = {Attention is Not All You Need: Pure Attention Loses Rank Doubly Exponentially with Depth},
  author    = {Yihe Dong and Jean-Baptiste Cordonnier and Andreas Loukas},
  booktitle = {Proceedings of the 38th International Conference on Machine Learning (ICML)},
  volume    = {139},
  pages     = {2793--2803},
  year      = {2021},
  publisher = {PMLR},
  url       = {https://proceedings.mlr.press/v139/dong21a.html}
}

@inproceedings{Gromov2024Ineffectiveness,
  title     = {The Unreasonable Ineffectiveness of the Deeper Layers},
  author    = {Andrey Gromov and Kushal Tirumala and Hassan Shapourian and Paolo Glorioso and Daniel A. Roberts},
  booktitle = {International Conference on Learning Representations (ICLR)},
  year      = {2025},
  note      = {Also available as arXiv:2403.17887},
  url       = {https://openreview.net/forum?id=XJ2ZqJ2s4d}
}

@inproceedings{Hron2020,
  title     = {Infinite Attention: {NNGP} and {NTK} for Deep Attention Networks},
  author    = {Jiri Hron and Yasaman Bahri and Jascha Sohl-Dickstein and Roman Novak},
  booktitle = {Proceedings of the 37th International Conference on Machine Learning (ICML)},
  volume    = {119},
  pages     = {4376--4386},
  year      = {2020},
  publisher = {PMLR},
  url       = {https://proceedings.mlr.press/v119/hron20a.html}
}

@inproceedings{Impagliazzo1995,
  title     = {A Personal View of Average-Case Complexity},
  author    = {Russell Impagliazzo},
  booktitle = {Proceedings of the 10th Annual Structure in Complexity Theory Conference},
  pages     = {134--147},
  year      = {1995},
  organization={IEEE},
  doi       = {10.1109/SCT.1995.514853}
}

@inproceedings{Jacot2018,
  title     = {Neural Tangent Kernel: Convergence and Generalization in Neural Networks},
  author    = {Arthur Jacot and Franck Gabriel and Cl{\'e}ment Hongler},
  booktitle = {Advances in Neural Information Processing Systems (NeurIPS)},
  volume    = {31},
  pages     = {8571--8580},
  year      = {2018},
  url       = {https://proceedings.neurips.cc/paper/2018/hash/5a4be1fa34e62bb8a6ec6b91d2462f5a-Abstract.html}
}

@inproceedings{Lee2018DNN,
  title     = {Deep Neural Networks as Gaussian Processes},
  author    = {Jaehoon Lee and Yasaman Bahri and Roman Novak and Sam Schoenholz and Jeffrey Pennington and Jascha Sohl-Dickstein},
  booktitle = {International Conference on Learning Representations (ICLR)},
  year      = {2018},
  url       = {https://openreview.net/forum?id=B1EA-M-0Z}
}

@article{Levin1986,
  title     = {Average Case Complete Problems},
  author    = {Leonid A. Levin},
  journal   = {SIAM Journal on Computing},
  volume    = {15},
  number    = {1},
  pages     = {285--286},
  year      = {1986},
  doi       = {10.1137/0215020}
}

@article{Merrill2022Saturated,
  title     = {Saturated Transformers are Constant-Depth Threshold Circuits},
  author    = {William Merrill and Ashish Sabharwal and Noah A. Smith},
  journal   = {Transactions of the Association for Computational Linguistics},
  volume    = {10},
  pages     = {843--856},
  year      = {2022},
  doi       = {10.1162/tacl_a_00492}
}

@inproceedings{Mirzadeh2024,
  title     = {{GSM}-Symbolic: Understanding the Limitations of Mathematical Reasoning in Large Language Models},
  author    = {Iman Mirzadeh and Keivan Alizadeh and Hooman Shahrokhi and Oncel Tuzel and Samy Bengio and Mehrdad Farajtabar},
  booktitle = {International Conference on Learning Representations (ICLR)},
  year      = {2025},
  note      = {Also available as arXiv:2410.05229},
  url       = {https://openreview.net/forum?id=u21fK4h5YQ}
}

@inproceedings{Orlin2013MaxFlows,
  title     = {Max Flows in {$O(nm)$} Time, or Better},
  author    = {James B. Orlin},
  booktitle = {Proceedings of the 45th Annual ACM Symposium on Theory of Computing (STOC)},
  pages     = {765--774},
  year      = {2013},
  doi       = {10.1145/2488608.2488705}
}

@book{Penrose2003,
  title     = {Random Geometric Graphs},
  author    = {Mathew Penrose},
  volume    = {5},
  series    = {Oxford Studies in Probability},
  publisher = {Oxford University Press},
  year      = {2003},
  isbn      = {9780198506263}
}

@article{Power2022Grokking,
  title     = {Grokking: Generalization Beyond Overfitting on Small Algorithmic Datasets},
  author    = {Alethea Power and Yuri Burda and Harri Edwards and Igor Babuschkin and Vedant Misra},
  journal   = {arXiv preprint arXiv:2201.02177},
  year      = {2022},
  url       = {https://arxiv.org/abs/2201.02177}
}

@article{Ramamoorthy2005Capacity,
  title     = {On the Capacity of Network Coding for Random Networks},
  author    = {Aditya Ramamoorthy and Jun Shi and Richard D. Wesel},
  journal   = {IEEE Transactions on Information Theory},
  volume    = {51},
  number    = {8},
  pages     = {2878--2885},
  year      = {2005},
  doi       = {10.1109/TIT.2005.851722}
}

@inproceedings{Schoenholz2016,
  title     = {Deep Information Propagation},
  author    = {Samuel S. Schoenholz and Justin Gilmer and Surya Ganguli and Jascha Sohl-Dickstein},
  booktitle = {International Conference on Learning Representations (ICLR)},
  year      = {2017},
  url       = {https://openreview.net/forum?id=H1W1UN9gg}
}

@article{Nair2026Softmax,
title={Softmax is \$1/2\$-Lipschitz: A tight bound across all \${\textbackslash}ell\_p\$ norms},
author={Pravin Nair},
journal={Transactions on Machine Learning Research},
issn={2835-8856},
year={2026},
url={https://openreview.net/forum?id=6dowaHsa6D},
note={}
}

@InProceedings{Edelman2022,
  title = 	 {Inductive Biases and Variable Creation in Self-Attention Mechanisms},
  author =       {Edelman, Benjamin L and Goel, Surbhi and Kakade, Sham and Zhang, Cyril},
  booktitle = 	 {Proceedings of the 39th International Conference on Machine Learning},
  pages = 	 {5793--5831},
  year = 	 {2022},
  editor = 	 {Chaudhuri, Kamalika and Jegelka, Stefanie and Song, Le and Szepesvari, Csaba and Niu, Gang and Sabato, Sivan},
  volume = 	 {162},
  series = 	 {Proceedings of Machine Learning Research},
  month = 	 {17--23 Jul},
  publisher =    {PMLR},
  url = 	 {https://proceedings.mlr.press/v162/edelman22a/edelman22a.pdf}
}

@inproceedings{Kool2019Attention,
    author       = {Wouter Kool and Herke van Hoof and Max Welling},
    title        = {Attention, Learn to Solve Routing Problems!},
    booktitle    = {7th International Conference on Learning Representations, {ICLR} 2019, New Orleans, LA, USA, May 6-9},
    year         = {2019}
}

@inproceedings{Xu2020NAR,
title={What Can Neural Networks Reason About?},
author={Keyulu Xu and Jingling Li and Mozhi Zhang and Simon S. Du and Ken-ichi Kawarabayashi and Stefanie Jegelka},
booktitle={International Conference on Learning Representations},
year={2020},
url={https://openreview.net/forum?id=rJxbJeHFPS}
}

@inproceedings{Xu2021NAR,
  author       = {Keyulu Xu and
                  Mozhi Zhang and
                  Jingling Li and
                  Simon Shaolei Du and
                  Ken{-}ichi Kawarabayashi and
                  Stefanie Jegelka},
  title        = {How Neural Networks Extrapolate: From Feedforward to Graph Neural
                  Networks},
  booktitle    = {9th International Conference on Learning Representations, {ICLR} 2021,
                  Virtual Event, Austria, May 3-7, 2021},
  publisher    = {OpenReview.net},
  year         = {2021},
  url          = {https://openreview.net/forum?id=UH-cmocLJC},
  bibsource    = {dblp computer science bibliography, https://dblp.org}
}

@inproceedings{DeLuca2025NAR,
author = {De Luca, Artur Back and Giapitzakis, George and Yang, Shenghao and Veli\v{c}kovi\'{c}, Petar and Fountoulakis, Kimon},
title = {Positional attention: expressivity and learnability of algorithmic computation},
year = {2025},
publisher = {JMLR.org},
booktitle = {Proceedings of the 42nd International Conference on Machine Learning},
articleno = {84},
numpages = {56},
location = {Vancouver, Canada},
series = {ICML'25}
}

@inproceedings{Bevilacqua2023NAR,
  author       = {Beatrice Bevilacqua and
                  Kyriacos Nikiforou and
                  Borja Ibarz and
                  Ioana Bica and
                  Michela Paganini and
                  Charles Blundell and
                  Jovana Mitrovic and
                  Petar Velickovic},
  editor       = {Andreas Krause and
                  Emma Brunskill and
                  Kyunghyun Cho and
                  Barbara Engelhardt and
                  Sivan Sabato and
                  Jonathan Scarlett},
  title        = {Neural Algorithmic Reasoning with Causal Regularisation},
  booktitle    = {International Conference on Machine Learning, {ICML} 2023, 23-29 July
                  2023, Honolulu, Hawaii, {USA}},
  series       = {Proceedings of Machine Learning Research},
  pages        = {2272--2288},
  publisher    = {{PMLR}},
  year         = {2023},
  url          = {https://proceedings.mlr.press/v202/bevilacqua23a.html},
  bibsource    = {dblp computer science bibliography, https://dblp.org}
}

@inproceedings{Rodionov2023NAR,
 author = {Rodionov, Gleb and Prokhorenkova, Liudmila},
 booktitle = {Advances in Neural Information Processing Systems},
 editor = {A. Oh and T. Naumann and A. Globerson and K. Saenko and M. Hardt and S. Levine},
 pages = {51663--51674},
 publisher = {Curran Associates, Inc.},
 title = {Neural Algorithmic Reasoning Without Intermediate Supervision},
 url = {https://proceedings.neurips.cc/paper_files/paper/2023/file/a2370db7c99791ad5d9f3ef48ad6d464-Paper-Conference.pdf},
 volume = {36},
 year = {2023}
}

@inproceedings{Georgiev2024NAR,
 author = {Georgiev, Dobrik and Wilson, JJ and Buffelli, Davide and Li\`{o}, Pietro},
 booktitle = {Advances in Neural Information Processing Systems},
 doi = {10.52202/079017-1059},
 editor = {A. Globerson and L. Mackey and D. Belgrave and A. Fan and U. Paquet and J. Tomczak and C. Zhang},
 pages = {33638--33667},
 publisher = {Curran Associates, Inc.},
 title = {Deep Equilibrium Algorithmic Reasoning},
 url = {https://proceedings.neurips.cc/paper_files/paper/2024/file/3b1675de6b49cc00084374213f8c38ae-Paper-Conference.pdf},
 volume = {37},
 year = {2024}
}

@misc{Yu2026NAR,
      title={FloydNet: A Learning Paradigm for Global Relational Reasoning}, 
      author={Jingcheng Yu and Mingliang Zeng and Qiwei Ye},
      year={2026},
      eprint={2601.19094},
      archivePrefix={arXiv},
      primaryClass={cs.LG},
      url={https://arxiv.org/abs/2601.19094}, 
}

@inproceedings{Vaswani2017Transformer,
author = {Vaswani, Ashish and Shazeer, Noam and Parmar, Niki and Uszkoreit, Jakob and Jones, Llion and Gomez, Aidan N. and Kaiser, \L{}ukasz and Polosukhin, Illia},
title = {Attention is all you need},
year = {2017},
isbn = {9781510860964},
publisher = {Curran Associates Inc.},
address = {Red Hook, NY, USA},
booktitle = {Proceedings of the 31st International Conference on Neural Information Processing Systems},
pages = {6000–6010},
numpages = {11},
location = {Long Beach, California, USA},
series = {NIPS'17}
}

@inproceedings{Anil2022LengthGeneralization,
 author = {Anil, Cem and Wu, Yuhuai and Andreassen, Anders and Lewkowycz, Aitor and Misra, Vedant and Ramasesh, Vinay and Slone, Ambrose and Gur-Ari, Guy and Dyer, Ethan and Neyshabur, Behnam},
 booktitle = {Advances in Neural Information Processing Systems},
 editor = {S. Koyejo and S. Mohamed and A. Agarwal and D. Belgrave and K. Cho and A. Oh},
 pages = {38546--38556},
 publisher = {Curran Associates, Inc.},
 title = {Exploring Length Generalization in Large Language Models},
 url = {https://proceedings.neurips.cc/paper_files/paper/2022/file/fb7451e43f9c1c35b774bcfad7a5714b-Paper-Conference.pdf},
 volume = {35},
 year = {2022}
}

@article{Su2021RoPE,
  title={RoFormer: Enhanced Transformer with Rotary Position Embedding},
  author={Jianlin Su and Yu Lu and Shengfeng Pan and Bo Wen and Yunfeng Liu},
  journal={ArXiv},
  year={2021},
  volume={abs/2104.09864},
  url={https://api.semanticscholar.org/CorpusID:233307138}
}

@book{Sipser2006CFG,
  author    = {Michael Sipser},
  title     = {Introduction to the Theory of Computation},
  edition   = {2},
  year      = {2006},
  publisher = {Cengage Learning}
}

@inproceedings{Wen2026RLVR,
title={Reinforcement Learning with Verifiable Rewards Implicitly Incentivizes Correct Reasoning in Base {LLM}s},
author={Xumeng Wen and Zihan Liu and Shun Zheng and Shengyu Ye and Zhirong Wu and Yang Wang and Zhijian Xu and Xiao Liang and Junjie Li and Ziming Miao and Jiang Bian and Mao Yang},
booktitle={The Fourteenth International Conference on Learning Representations},
year={2026},
url={https://openreview.net/forum?id=jGbRWwIidy}
}
\clearpage
%%%%%%%%%%%%%%%%%%%%%%%%%%%%%%%%
% APPENDIX
%%%%%%%%%%%%%%%%%%%%%%%%%%%%%%%%
\appendix

%%%%%%%%%%%%%%%%%%%%%%%%%%%%%%%%
%%%%%%%%%%%%%%%%%%%%%%%%%%%%%%%%
%%%%%%%%%%%%%%%%%%%%%%%%%%%%%%%%

\section{Finite-width Transformers in Detail}

%%%%%%%%%%%%%%%%%%%%%%%%%%%%%%%%
%%%%%%%%%%%%%%%%%%%%%%%%%%%%%%%%
%%%%%%%%%%%%%%%%%%%%%%%%%%%%%%%%

\subsection{Setup}
\label{app:transformer_setup}

As stated, we focus on decoder-only transformers with softmax self-attention and no causal masking. We consider a multi-layer, multi-head transformer with $L$ transformer blocks and $H$ heads in each. We define each transformer block as follows. For block input $Z'(X) \in \bR^{T \times \dmod}$, we define the pre-activation
\begin{align}
\label{eq:transformer}
    Z(X) = \frac{1}{\sqrt{H}} \sum_{h=1}^H  A^{h}(X) V^{h}(X) \in \bR^{T \times \dmod}   
\end{align}
with attention weights
\begin{align}
    A^{h}(X) = \Softmax\left(\frac{1}{\sqrt{d_k}} Q^{h}(X) K^{h}(X)^\top\right) \in \bR^{T \times T},
\end{align}
where 
\begin{align}
    & Q^{h}(X)=Z'(X)W^{h}_Q \in \bR^{T \times d_k}, \\
    & K^{h}(X)=Z'(X)W^{h}_K \in \bR^{T \times d_k}, \\ 
    & V^{h}(X)=Z'(X)W^{h}_V \in \bR^{T \times \dmod},
\end{align}
for
\begin{align}
    W^{h}_Q,W^{h}_K \in \bR^{\dmod \times d_k}, \quad W^{h}_V \in \bR^{\dmod \times \dmod},
\end{align}
with entries sampled
\begin{align}
    (W^{h}_Q)_{ij},(W^{h}_K)_{ij},(W^{h}_V)_{ij} \overset{\mathrm{iid}}{\sim} \cN(0,1/\dmod).
\end{align}
In particular, for $a \in [T]$,
\begin{align}
    \bm{z}_a(X) = \frac{1}{\sqrt{H}} \sum_{h=1}^H \sum_{b=1}^{T} A^{h}_{ab}(X) V^{h}_b(X) \in \bR^{\dmod}.
\end{align}

Next, we apply LayerNorm to the pre-activation per token 
\begin{align}
\label{eq:layer_norm}
    \LayerNorm(\bfz_a)
    =
    \bm{\gamma}_{\mathrm{LN}} \odot
    \frac{\bfz_a - \frac{1}{\dmod}(\bfz_a^\top \bm{1})\bm{1}}
    {\sqrt{\norm{\bfz_a - \frac{1}{\dmod}(\bfz_a^\top \bm{1})\bm{1}}^2/\dmod + \epsilon}}
    + \bm{\beta}_{\mathrm{LN}}
\end{align}
with learned parameters $\bm{\beta}_{\mathrm{LN}}, \bm{\gamma}_{\mathrm{LN}} \in \bR^{\dmod}$, and $\bm{1} \in \bR^{\dmod}$ denotes the all-ones vector, while $\odot$ denotes the Hadamard (entrywise) product. Adding multilayer perceptron (MLP) nonlinearity, the output of the transformer block at the $a$-th token is given by
\begin{align}
\label{eq:mlp_layer}
    \mlp(\bfz_a) = W_{\mlp}\phi(\LayerNorm(\bfz_a))+\bm{b}_{\mlp}
\end{align}
for an activation function $\phi$ applied entry-wise and learned weights $W_{\mlp} \in \bR^{\dmod \times \dmod}$ and $\bm{b}_{\mlp} \in \bR^{\dmod}$ drawn via
\begin{align}
\label{eq:tranformer_mlp_draws}
    W_{ij} \overset{\mathrm{i.i.d.}}{\sim} \cN\!\left(0,\frac{\sigma^2_{W}}{d_{\mathrm{model}}}\right), \qquad b_i \overset{\mathrm{i.i.d.}}{\sim} \cN(0,\sigma^2_{\bm b}).
\end{align}
In this paper, we will assume the activation function $\phi$ is bounded by linear growth, e.g., ReLU or GeLU.

This setup is intentionally simplified: we omit the standard two-layer feed-forward network (FFN), using instead a single point-wise nonlinearity followed by a linear map. We also omit residual connections and the output projection $W_O$, combining heads by averaging instead, as in Eq.~\eqref{eq:transformer}. Then
\begin{align}
\theta_{\mathrm{block}}
=
\Bigl(
\bigl(W_Q^h, W_K^h, W_V^h\bigr)_{h=1}^H,\;
\bm{\beta}_{\mathrm{LN}},\;
\bm{\gamma}_{\mathrm{LN}},\;
W_{\mlp},\;
\bm{b}_{\mlp}
\Bigr)
\end{align}
denotes the collection of all learned parameters of a single transformer block in our simplified architecture.

As input to the first transformer block, $Z'(X)$ enters attention as the model-space representation of the input tokens. Including a learned embedding from raw input dimension into the model's internal representation space, we have
\begin{align}
    Z'(X) = XW_{\emb}
\end{align}
for learned weights $W_{\emb} \in \bR^{d \times \dmod}$, where we can replace $X$ with a positionally encoded $\mathrm{PE}(X)$, if needed. 

%%%%%%%%%%%%%%%%%%%%%%%%%%%%%%%%
%%%%%%%%%%%%%%%%%%%%%%%%%%%%%%%%
%%%%%%%%%%%%%%%%%%%%%%%%%%%%%%%%

\subsection{Inference-time Finite-width Complexity}
\label{app:inference_time_complexity}

\begin{lemma}
\label{lem:finite_width_complexity}
    The inference-time complexity of a finite-width transformer is $O(LH\dmod T^2)$.
\end{lemma}

\begin{proof}
First, note that the number of floating-point operations (FLOPs) involved in multiplying matrices $(m,k)\times(k,n)$ is $\approx 2mkn$. Therefore, calculating $Q(X),K(X)$ is $\approx 2T\dmod d_k$, and calculating $V(X)$ is $\approx 2T\dmod^2$. Thus, calculating $Q(X)K(X)^\top$ is $\approx 2T^2 d_k$. Calculating $\Softmax$ is $\Theta(T^2)$, while multiplying by $V(X)$ is $\approx 2T^2 \dmod$. Thus, for the attention layer in a transformer block, we get
\begin{align}
    \flops_{\mathrm{attn}} \approx H(4T\dmod d_k + 2T\dmod^2 + 2T^2 d_k + \Theta(T^2) + 2T^2 \dmod).
\end{align}
As for the $\mlp$ block, $\flops_{\mlp} \approx 2T\dmod^2$. And together, we get
\begin{align}
    \flops_{\mathrm{block}} = O(H\left[T\dmod d_k + T\dmod^2 + T^2(d_k+\dmod)\right])
\end{align}
For a typical choice $d_k=\dmod/H$, we get
\begin{align}
    \flops_{\mathrm{block}} = O\left(H\left[T\dmod^2 + T^2\dmod\right]\right)
\end{align}
And for a  transformer with $L$ blocks
\begin{align}
    \flops_{\mathrm{trans}} = O\left(LH\left[T\dmod^2 + T^2\dmod\right]\right),
\end{align}
which completes the proof.
\end{proof}

%%%%%%%%%%%%%%%%%%%%%%%%%%%%%%%%
%%%%%%%%%%%%%%%%%%%%%%%%%%%%%%%%
%%%%%%%%%%%%%%%%%%%%%%%%%%%%%%%%

\subsection{PAC-style Capture}
\label{app:pac_style_capture}

One can naturally extend Definition~\ref{def:capture} to a conditional PAC-style algorithmic 
combinatorial task capture, in which the step-one trained model of our two-step training protocol is treated as fixed.
Fix a transformer architecture, initialization distribution, training procedure,
hyperparameter choices, and size-sampling schedules. For a predictor $h:\cX\to\cY$,
denote its size-$T$ population error by
\begin{align}
    R_T(h) =
    \Pr\nolimits_{X\sim\mu_{X,T}}
    \left[
        \dist_{\cY}(h(X),g(X))\ge \Delta/3
    \right].
\end{align}
We say that the transformer conditionally PAC-captures the combinatorial task
$g:\cX\to\cY$ if for every sufficiently small $\delta>0$ and every $\beta>0$
there exist
$T_0=T_0(\delta,\beta)$,
$P_0=P_0(\delta,\beta)$, and
$C_0=C_0(\delta,\beta)$ such that the following holds.

Fix an initial training set $\cD_0$ of $P_0$ labeled samples from
$\mu_{X,1},\dots,\mu_{X,T_0}$, drawn according to the fixed size-sampling schedule,
and fix the randomness of the initial training stage. Let $f_{T_0}$ be the resulting
step-one trained predictor. Then, for every $T\ge T_0$, if $\cD_T$ is a fresh
adaptation set of at most $C_0\log(T/T_0)$ labeled samples drawn from
$\mu_{X,1},\dots,\mu_{X,T}$ according to the fixed size-sampling schedule, and
$f_T$ is obtained by fine-tuning $f_{T_0}$ on $\cD_T$, then
\begin{align}
\Pr\nolimits_{\cD_T,\mathrm{train}_T}
    \left[
        R_T(f_T)<\delta
        \,\middle|\,
        \cD_0,\mathrm{train}_0
    \right]
    \ge 1-\beta .
\end{align}
The conditional notion treats $f_{T_0}$ as fixed and gives a PAC-style guarantee
only for the logarithmic adaptation stage. A full PAC-style combinatorial task capture would then require
\begin{align}
\Pr\nolimits_{\cD_0,\mathrm{train}_0}\left[\Pr\nolimits_{\cD_T,\mathrm{train}_T}
    \left[
        R_T(f_T)<\delta
        \,\middle|\,
        \cD_0,\mathrm{train}_0
    \right]
    \ge 1-\beta\right] \ge 1-\alpha, \quad T\ge T_0.
\end{align}
with a sufficiently small $\alpha>0$ folded in. It controls the probability that the initial training stage fails to
produce a base model from which logarithmic adaptation succeeds, while $\beta$
controls the failure probability of the secondary adaptation stage.

%%%%%%%%%%%%%%%%%%%%%%%%%%%%%%%%
%%%%%%%%%%%%%%%%%%%%%%%%%%%%%%%%
%%%%%%%%%%%%%%%%%%%%%%%%%%%%%%%%

\subsection{Attention Scores}
\label{app:attention_scores}

The attention scores for the finite-width transformer were defined in Eq.~\eqref{eq:attention_scores}. In the infinite-width limit, the attention score matrix is a centered multivariate Gaussian whose covariance can be calculated directly.

\CovarianceFactorization*

\begin{proof}
Utilizing iterated expectation, 
\begin{align}
    \bE\left[S_{ac}^h(X_1) S^{h}_{be}(X_2)\right] = \bE_{Z'}\!\left[\bE_{(W^h_Q,W^{h}_K)}[S_{ac}^h(X_1) S^{h}_{be}(X_2)| Z']\right].
\end{align}
we consider the finite-width conditional overlap, 
\begin{align*}
    & \bE_{(W^h_Q,W^{h}_K)}\left[S_{ac}^h(X_1) S^{h}_{be}(X_2)\right|\!\left.Z'\right] \\
    &= \frac{1}{d_k} \bE_{(W^h_Q,W^{h}_K)}\left[{\bm{z'}_a}(X_1)^{\top} W^h_Q (W^h_K)^{\top} \bm{z'}_{c}(X_1) \cdot {\bm{z'}_b}(X_2)^{\top} W^h_Q (W^h_K)^{\top} \bm{z'}_{e}(X_2)\right|\!\left.Z'\right]\\
    &= \frac{1}{d_k} \sum_{ijk} \sum_{lmn} [{\bm{z'}_a}]_i [{\bm{z'}_c}]_j [{\bm{z'}_b}]_l [{\bm{z'}_e}]_m \bE_{(W^h_Q,W^{h}_K)}\left[ [W^h_Q]_{ik}[W^h_K]_{jk} [W^h_Q]_{ln}[W^h_K]_{mn} \right]\\
    &= \frac{1}{d_k} \sum_{ijk} \sum_{lmn} [{\bm{z'}_a}]_i [{\bm{z'}_c}]_j [{\bm{z'}_b}]_l [{\bm{z'}_e}]_m \bE_{W^h_Q}\left[[W^h_Q]_{ik}[W^h_Q]_{ln}\right] \bE_{W^{h}_K}\left[[W^h_K]_{jk}[W^h_K]_{mn}\right].
\end{align*}
Independence enforces $i=l$, $j=m$, and $k=n$, thus
\begingroup
\allowdisplaybreaks
\begin{align*}
    & \bE_{(W^h_Q,W^{h}_K)}\left[S_{ac}^h(X_1) S^{h}_{be}(X_2)\right|\!\left.Z'\right] \\
    &= \frac{1}{d_k} \sum_{ij=1}^{\dmod} [{\bm{z'}_a}]_i [{\bm{z'}_b}]_i [{\bm{z'}_c}]_j [{\bm{z'}_e}]_j \left(\sum_{k=1}^{d_k} \bE_{W^h_Q}\left[[W^h_Q]_{ik}^2\right] \bE_{W^{h}_K}\left[[W^h_K]_{jk}^2\right]\right) \\
    &= \frac{1}{d_k} \sum_{ij=1}^{\dmod} [{\bm{z'}_a}]_i [{\bm{z'}_b}]_i [{\bm{z'}_c}]_j [{\bm{z'}_e}]_j \left(\frac{d_k}{\dmod^2}\right) \\
    &= \left(\frac{{\bm z'_a}(X_1)^\top{\bm z'_{b}}(X_2)}{\dmod}\right)  \left(\frac{\bm{z'}_{c}(X_1)^\top \bm{z'}_e(X_2)}{\dmod}\right),
\end{align*}
\endgroup
where we used $(W^{h}_Q)_{ik},(W^{h}_K)_{jk} \overset{\mathrm{iid}}{\sim} \cN(0,1/\dmod)$.

In the infinite-width limit, the attention score matrix $S^{h}(X)$ becomes a centered multivariate Gaussian with respect to the ensemble $(W_K^h,W_Q^h)$ as a direct consequence of the central limit theorem (CLT). The overlap concentrates to the covariance kernel
\begin{align}
\lim_{\dmod \rightarrow \infty} \bE_{Z'}\left[\bE_{(W^h_Q,W^{h}_K)}\left[S_{ac}^h(X_1) S^{h}_{be}(X_2)\right|\!\left.Z'\right]\right] = \Cov'_{ab}(X_1,X_2) \Cov'_{ce}(X_1,X_2),   
\end{align}
as in Eq.~\eqref{eq:covariance}.  
\end{proof}

%%%%%%%%%%%%%%%%%%%%%%%%%%%%%%%%
%%%%%%%%%%%%%%%%%%%%%%%%%%%%%%%%
%%%%%%%%%%%%%%%%%%%%%%%%%%%%%%%%

\subsection{Complexity of Score Draws}
\label{app:MC}

The MC estimation of Eq.~\eqref{eq:infinite_covariance_update} involves sampling scores $S_{ac}(X_1)$ and $S_{be}(X_2)$ for any $a,c,b,e \in [T]$. By Proposition~\ref{prop:CovarianceFactorization}, the attention score matrix $S(X)$ in the infinite-width limit is a centered multivariate Gaussian with known covariance, and we can work out the computational cost of drawing a pair $(S(X_1),S(X_2))$, for fixed $X_1,X_2 \in \bR^{T \times d}$, accordingly.

\begin{proposition}
\label{prop:score_sample_complexity}
    Drawing a $2T^2$ sample $(S_{ac}(X_1),S_{be}(X_2))_{a,c,b,e \in [T]}$ admits a computational cost of $O(T^3)$.
\end{proposition}

\begin{proof}
We utilize the Kronecker tensor structure of the score matrix covariance as well as conditioning in multivariate Gaussians to analyze its computational complexity. First, we consider $S(X_1) \in \bR^{T \times T}$. Then 
\begin{align}
    S(X_1) \sim \cN(0,\Theta(X_1,X_1))
\end{align}
with $\Theta(X_1,X_1) = \Sigma'(X_1,X_1) \otimes \Sigma'(X_1,X_1)$ as in Eq.~\eqref{eq:Factorization}. Note that $\Sigma'(X_1,X_1)\succeq 0$ by definition and consequently $\Theta(X_1,X_1)\succeq 0$. Given a Cholesky decomposition $\Sigma'(X_1,X_1)=L_1L_1^\top$, we can decompose $\Theta(X_1,X_1) = (L_1\otimes L_1)(L_1\otimes L_1)^\top$. An $S(X_1)$ sample can then be generated at a cost of $O(T^3)$ via $S(X_1) = L_1UL_1^\top$ for $U_{ab} \overset{\text{i.i.d.}}{\sim} \cN(0,1)$.

Next, we apply Bayes' rule
\begin{align}
    p_{(S(X_1),S(X_2))} = p_{S(X_1)} p_{S(X_2) \mid S(X_1)}.
\end{align}
The joint covariance of  $(S(X_1),S(X_2))$ is the block matrix
\begin{align}
    \left(\begin{array}{cc}
        \Theta_{11} & \Theta_{12} \\
         \Theta_{21} & \Theta_{22}
    \end{array}\right) = 
    \left(\begin{array}{cc}
        \Sigma'_{11} \otimes \Sigma'_{11} & \Sigma'_{12} \otimes \Sigma'_{12} \\
        \Sigma'_{21} \otimes \Sigma'_{21} & \Sigma'_{22} \otimes \Sigma'_{22}
    \end{array}\right)\succeq 0
\end{align}
where $\Sigma'_{ij} \equiv \Sigma'(X_i,X_j)$. The conditional distribution of $S(X_2)$ given $S(X_1)$ is known \citep[Sec. 9.3]{vonMises1964Probability} to be
\begin{align}
    S(X_2) | S(X_1) &\sim \cN\left(M_{\mathrm{cond}},\Theta_{\mathrm{cond}}\right)\\
    M_{\mathrm{cond}} &= \Theta_{21}\Theta_{11}^{-1}\mathrm{vec}(S(X_1)) \\
    \Theta_{\mathrm{cond}} &= \Theta_{22}-\Theta_{21}\Theta_{11}^{-1}\Theta_{12} \succeq 0
\end{align}
which involves calculating
\begin{align}
    M_{\mathrm{cond}} = \Sigma'_{21}(\Sigma'_{11})^{-1} S(X_1) (\Sigma'_{21}(\Sigma'_{11})^{-1})^\top
\end{align}
and $\Sigma'_{21}(\Sigma'_{11})^{-1}\Sigma'_{21}$ for $\Theta_{\mathrm{cond}}$, again, at a cost of $O(T^3)$.\footnote{If $\Sigma'_{11}$ is not positive definite (PD), we can use the Moore-Penrose pseudo-inverse and the conditional becomes degenerate.}

To draw a sample from the conditioned Gaussian, we will assume $\Sigma'_{11},\Sigma'_{22}\succ 0$. We can then calculate 
\begin{align}
    H := L_2^{-1}\Sigma'_{21}(\Sigma'_{11})^{-1}\Sigma'_{21} (L_2^\top)^{-1}
\end{align}
for a Cholesky decomposition $\Sigma'_{22}=L_2L_2^\top$ at a cost of $O(T^3)$. Then we have
\begin{align*}
     \left((L_2^\top)^{-1} \otimes (L_2)^{-1}\right) \Theta_{\mathrm{cond}} \left((L_2)^{-1} \otimes (L_2^\top)^{-1}\right) = I \otimes I - H \otimes H.
\end{align*}
Since the LHS is PSD, we have $I-H\succeq 0$. Consequently, if we diagonalize $H = Q\Lambda Q^{-1}$ with $\Lambda = \mathrm{diag}(\bm{\lambda})$ and $\bm{\lambda}=(\lambda_1,\ldots,\lambda_T)$, then we can generate a sample at a cost of $O(T^3)$ with 
\begin{align}
    S(X_2) | S(X_1) &= M_{\mathrm{cond}} + L_2 Q \Delta Q^\top L_2^\top \\
    \Delta &:= \sqrt{\bm{1}\bm{1}^\top-\bm{\lambda}\bm{\lambda}^\top} \odot Q^{-1} U (Q^\top)^{-1}
\end{align}
for $U_{ab} \overset{\text{i.i.d.}}{\sim} \cN(0,1)$, where $\odot$ is matrix element-wise product. To conclude, we obtain a sample of $(S(X_1),S(X_2))$ at an overall cost of $O(T^3)$.
\end{proof}

%%%%%%%%%%%%%%%%%%%%%%%%%%%%%%%%
%%%%%%%%%%%%%%%%%%%%%%%%%%%%%%%%
%%%%%%%%%%%%%%%%%%%%%%%%%%%%%%%%

\section{Kernel Evaluation}
\label{app:kernel_evaluation}

%%%%%%%%%%%%%%%%%%%%%%%%%%%%%%%%
%%%%%%%%%%%%%%%%%%%%%%%%%%%%%%%%
%%%%%%%%%%%%%%%%%%%%%%%%%%%%%%%%

\subsection{Propagating Through the Attention Layer}
\label{app:attention_layer_propagation}

\AttentionCovarianceUpdate*

\begin{proof}
In finite width, the covariance update across the attention layer requires computing the expectation $\bE[\bm{z}_a(X_1)^\top \bm{z}_b(X_2)]$ for all pairs $a,b \in [T]$, which involves averaging over the random hidden representations $Z'=(Z'(X_1),Z'(X_2))$ and the random weights $W=(W^h_Q,W^{h}_K,W^{h}_V)$. 
By iterated expectation, we have
\begin{align}
    \bE\left[\bm{z}_a(X_1)^{\top} \bm{z}_b(X_2)\right] = \bE_{Z'}\!\left[\bE_{W}[\bm{z}_a(X_1)^{\top} \bm{z}_b(X_2)| Z']\right].
\end{align}
The inner expectation is straightforward
\begingroup
\allowdisplaybreaks
\begin{align}
& \bE_{W}[\bm{z}_a(X_1)^{\top} \bm{z}_b(X_2)|Z'] = \nonumber\\
& = \frac{1}{H} \sum_{h,h'=1}^H \sum_{c,e=1}^T \bE_{W}[A^h_{ac}(X_1) A^{h'}_{be}(X_2) \bm{z'}_c(X_1)^\top W^h_{V} W^{h'\top}_{V} \bm{z'}_e(X_2) |Z'] \nonumber \\
& = \frac{1}{H} \sum_{h=1}^H \sum_{c,e=1}^T \bm{z'}_c(X_1)^\top \bm{z'}_e(X_2) \bE_{(W^h_Q,W^{h}_K)}[A^h_{ac}(X_1) A^{h}_{be}(X_2)  |Z'], \label{eq:conditional_finite_covariance}
\end{align}
\endgroup
where, in the second equality, we used $(W^{h}_V)_{ij} \overset{\mathrm{iid}}{\sim} \cN(0,1/\dmod)$, so that $\bE[W^h_{V} W^{h'\top}_{V}] = \delta_{hh'}I_{\dmod}$. Therefore,
\begin{multline}
    \frac{1}{\dmod} \bE\left[\bm{z}_a(X_1)^{\top} \bm{z}_b(X_2)\right] =\\  \sum_{c,e=1}^T \bE_{Z'}\!\left[\frac{\bm{z'}_c(X_1)^\top \bm{z'}_e(X_2)}{\dmod} \ \bE_{(W^h_Q,W^{h}_K)}\left[\frac{1}{H} \sum_{h=1}^H A^h_{ac}(X_1) A^{h}_{be}(X_2) \right.\left|Z' \vphantom{\sum_{h=1}^H}\right]\right].
\end{multline}

We assume the incoming token-to-token covariance matrix $\Cov'(X_1,X_2)$ is given. Taking the limit $\dmod,H \to \infty$, the normalized overlaps in $Z'$ may be replaced by the deterministic $\Cov_{ce}'(X_1,X_2)$, and the head average converges to expectation under the limiting score law determined by Proposition~\ref{prop:CovarianceFactorization}. Thus, we have
\begin{align}
\Cov_{ab}(X_1,X_2) &= \sum_{c,e=1}^T \Cov_{ce}'(X_1,X_2) \bE_{(S(X_1),S(X_2))}\left[A_{ac}(X_1) A_{be}(X_2)\right] \nonumber\\
&= \bE_{(S(X_1),S(X_2))}\left[A_{a\bullet}(X_1) \Cov'(X_1,X_2) A_{\bullet b}(X_2)^\top\right]
\end{align}
which completes the proof.
\end{proof}

\begin{remark}[First transformer block]
In the first transformer block $Z'(X)$ is the embedded input field in the model's internal representation space, which may include positional encoding (see App.~\ref{app:transformer_setup}). This case can be analyzed similarly, but is slightly more cumbersome, since it requires working directly with the induced query, key, and value fields $Q^h(X), K^h(X)$, and $V^h(X)$, rather than only with the recursive hidden field notation used for subsequent blocs.
\end{remark}

%%%%%%%%%%%%%%%%%%%%%%%%%%%%%%%%
%%%%%%%%%%%%%%%%%%%%%%%%%%%%%%%%
%%%%%%%%%%%%%%%%%%%%%%%%%%%%%%%%

\subsection{Propagating Through the Transformer Block}
\label{app:transformer_block_propagation}

We fix a transformer block in the NNGP or NTK limit and assume the incoming token-to-token covariance matrix $\Cov'(X_1,X_2) \in \bR^{T \times T}$ of Eq.~\eqref{eq:covariance} or the incoming NTK of Eq.~\eqref{eq:ntk} are given. We assume that the coefficients $\alpha_\nu$ of the kernel predictor induced by the propagated covariance $\Cov(X_1,X_2)$ and NTK $\Theta(X_1,X_2)$ satisfy the bounded-weight assumption of Remark~\ref{rem:bounded_kernel_weights}. In our setting, we also assume the activation function $\phi$ is bounded by linear growth (e.g., ReLU or GeLU).

Because standard transformer architectures employ LayerNorm (Eq.~\eqref{eq:layer_norm}), the input covariance matrix is bounded
\begin{align}
\label{eq:input_covariance_bounded}
    \norm{\Cov'(X_1,X_2)}_{\max} = \max_{a,b} \abs{\Cov'_{ab}(X_1,X_2)} = O(1).
\end{align}
Unlike the covariance bound, the NTK bound does not follow directly from
LayerNorm, since $\Theta'$ depends on derivatives with respect to all previous
parameters. However, for fixed depth, 
\begin{align}
\label{eq:input_ntk_bounded}
    \norm{\Theta'(X_1,X_2)}_{\max} = \max_{a,b} \abs{\Theta'_{ab}(X_1,X_2)} = O(1)
\end{align}
follows
by induction from the NTK recursion relation of \citep[Theorem 18]{Hron2020},
\begin{align}
\label{eq:ntk_recursion}
    \Theta_{ab} &= 2\Cov_{ab} + \bE_{(S_1,S_2)}\!\left[A_1\Theta'A_2^\top \right]_{ab} \nonumber\\
    &+ \left(2\Cov'_{ab}+\Theta'_{ab}\right)\bE_{(S_1,S_2)}\!\left[\mathrm{Tr}( (J_SA_1^a)^\top \Cov' J_SA_2^b (\Cov')^\top) \right] \nonumber\\
    &+ \Cov'_{ab} \bE_{(S_1,S_2)}\!\left[\mathrm{Tr}( (J_SA_1^a)^\top \Cov' J_SA_2^b (\Theta')^\top) \right],
\end{align}
where we used
\begin{align}
    & A_j \equiv A(X_j) \in \bR^{T \times T}, \quad A_j^a = (A_j)_{a\bullet} \in \bR^T \nonumber\\
    & S_j \equiv S(X_j) \in \bR^{T \times T}, \quad j=1,2
\end{align}
and $J_SA_j^a \in \bR^{T \times T}$ denote the Jacobian with respect to the scores $S$. We omit the pair $(X_1,X_2)$ from input and output kernels and attention scores for better readability.

\BlockPropagationComplexity*

\begin{proof}
It is sufficient to show that the computational complexity of evaluating the MC estimator of the NNGP / NTK kernel propagated through the attention layer is $O(N_{\mtc}T^3)$, with MC error scaling as $O(N_{\mtc}^{-1/2})$. We begin with the MC estimator of the propagated token-to-token covariance. Using Proposition~\ref{prop:AttentionCovarianceUpdate}, 
\begin{align}
\label{eq:mc_covariance_estimation}
    \Cov(X_1,X_2) \approx \frac{1}{N_{\mtc}} \sum_{n=1}^{N_{\mtc}}A(X_1)^{(n)}\Cov'(X_1,X_2)A(X_2)^{(n)\top}.
\end{align}
For each sample, we first draw $2T^2$ score samples $(S_{ac}(X_1),S_{be}(X_2))_{a,c,b,e \in [T]}$. By Proposition~\ref{prop:score_sample_complexity}, this can be done in $O(T^3)$ computational steps. Producing the attention matrices $A(X_1),A(X_2) \in \bR^{T \times T}$ requires additional $O(T^2)$ steps. Multiplying $A(X_1)^{(n)} \Cov'(X_1,X_2) A(X_2)^{(n)\top}$ requires $O(T^3)$ additional operations. The cost of computing a summand for a single MC sample is thus $O(T^3)$, dominated by efficient Cholesky decomposition and matrix multiplication, hence an overall $O(N_{\mtc}T^3)$. 

Next, we consider the MC estimator of the propagated NTK. 
By the recursion relation of Eq.~\eqref{eq:ntk_recursion}, the computational complexity of estimating $\Theta \in \bR^{T \times T}$ using MC integration is $O(N_{\mtc}T^3)$ for its first two summands, as before. To work out the trace terms, we note that 
\begin{align}
    J_SA_j^a = \diag(A_j^a) - A_j^a(A_j^a)^\top,
\end{align}
by \cite[Lemma 2]{Nair2026Softmax}, and, consequently,
\begin{align}
\label{eq:trace_attention_jacobians}
    \mathrm{Tr}( (J_SA_1)^\top M J_SA_2 N^\top) &= A_1(M\odot N)A_2^\top - A_1\left((MA_2^\top)\odot(NA_2^\top)\right) \nonumber\\
    &- \left((A_1M)\odot(A_1N)\right)A_2^\top + \left(A_1MA_2^\top\right)\odot\left(A_1NA_2^\top\right),
\end{align}
for any $M,N\in\bR^{T \times T}$.
Applying Eq.~\eqref{eq:trace_attention_jacobians} once to $M=N=\Cov'$ and twice to $M=\Cov',N=\Theta'$, we get a computational complexity of $O(T^3)$ for the trace terms in Eq.~\eqref{eq:ntk_recursion} simultaneously over all $a,b\in[T]$, and hence $O(N_{\mtc}T^3)$ for the MC estimator with $N_{\mtc}$ samples.

Kernel propagation through deep FCNs (App.~\ref{app:FCN}) was considered in \citep{Lee2018DNN}. In our setting, the analogous covariance propagation rule is obtained by replacing the activation map $\phi$ with
$\phi \circ \LayerNorm$. As we pass the covariance through the MLP layer
(Eq.~\ref{eq:mlp_layer}),
\begin{multline}
\label{eq:mlp_propagation}
    \Cov^{\mathrm{out}}_{ab}(X_1,X_2)
    = \sigma_{\bm{b}}^{2} + \\
    \lim_{\dmod\to\infty}
    \frac{\sigma_W^2}{\dmod}
    \bE_{\left(\bm z_{a}(X_1),\bm z_{b}(X_2)\right)\sim\cN(0,k)}
    \left[
        \phi\left(\LayerNorm(\bm{z}_{a}(X_1))\right)^\top
        \phi\left(\LayerNorm(\bm{z}_{b}(X_2))\right)
    \right],
\end{multline}
where \(\sigma_W,\sigma_{\bm b}\) are given in Eq.~\eqref{eq:tranformer_mlp_draws}
and \(\cN(0,k)\) is the bivariate Gaussian distribution with the \(2\times 2\)
covariance matrix \(k\) defined via
\begin{align}
\label{eq:mlp_2x2_covariance}
    k_{aa} = \Cov_{aa}(X_1,X_1),\quad
    k_{ab}=k_{ba}=\Cov_{ab}(X_1,X_2),\quad
    k_{bb}=\Cov_{bb}(X_2,X_2).
\end{align}
The corresponding NTK propagation through the MLP has the FCN form of \citet[Theorem~1]{Jacot2018}, with the scalar activation derivative kernel replaced by the Jacobian trace of $\phi\circ\LayerNorm$,
\begin{multline}
\label{eq:mlp_ntk_propagation}
    \Theta^{\mathrm{out}}_{ab}(X_1,X_2)
    = \Cov^{\mathrm{out}}_{ab}(X_1,X_2) + \\
    \Theta_{ab}(X_1,X_2)
    \lim_{\dmod\to\infty}
    \frac{\sigma_W^2}{\dmod}
    \bE_{\left(\bm z_a,\bm z_b\right)\sim\cN(0,k)}
    \left[
        \mathrm{Tr}\left(
            J_{\bm z_a}\left(\phi\circ\LayerNorm\right)^\top
            J_{\bm z_b}\left(\phi\circ\LayerNorm\right)
        \right)
    \right].
\end{multline}
Here \(J_{\bm z_a}(\phi\circ\LayerNorm)\in\bR^{\dmod\times \dmod}\) denotes the Jacobian with respect to the coordinates of \(\bm z_a\), namely
\begin{align}
    \left[J_{\bm z_a}(\phi\circ\LayerNorm)\right]_{ij}
    =
    \frac{\partial [\phi(\LayerNorm(\bm z_a))]_i}{\partial z_{a j}} .
\end{align}
Since the MLP acts token-wise,
neither Eq.~\eqref{eq:mlp_propagation} nor Eq.~\eqref{eq:mlp_ntk_propagation}
mix sequence indices. Thus, evaluating the MLP covariance and NTK updates for all \(a,b\in[T]\) costs \(O(T^2)\), up to a \(T\)-independent cost for the Gaussian expectations, and is dominated by the \(O(T^3)\) attention-layer
computation.

Each MC sample in Eq.~\eqref{eq:mc_covariance_estimation}  has uniformly bounded variance. Since the attention weights are nonnegative and row-stochastic, each sample entry is a convex combination of the entries of $\Cov'(X_1,X_2)$. Under LayerNorm, the $\Cov'(X_1,X_2)$ entries are bounded by a $T$-independent constant (Eq.~\eqref{eq:input_covariance_bounded}). Hence, the standard MC error for the propagated covariance estimate scales as $O(N_{\mtc}^{-1/2})$. 

To control the MC variance of the NTK estimator, we use the entrywise
$L_1$-bound for the softmax Jacobian from Lemma~\ref{lem:softmax_norms}(iv). Consequently, 
\begin{align}
\label{eq:trace_bound}
    \abs{\mathrm{Tr}\left((J_SA_1^a)^\top M J_SA_2^b N^\top\right)}
    &\le
    \sum_{c_1,c_2,e_1,e_2}
    \abs{M_{c_1c_2}}\abs{N_{e_1e_2}}
    \abs{(J_SA_1^a)_{c_1e_1}}
    \abs{(J_SA_2^b)_{c_2e_2}} \nonumber\\
    &\le 4\norm{M}_{\max}\norm{N}_{\max}
\end{align}
As noted, $\norm{\Cov'}_{\max}=O(1)=\norm{\Theta'}_{\max}$, and we get that the trace terms are uniformly bounded entrywise. Hence, here, as well, the standard MC error for the propagated NTK estimate scales as $O(N_{\mtc}^{-1/2})$. 

Since Eq.~\eqref{eq:kernel_predictor} is a weighted sum of $P$ independent MC-estimated kernel entries, each with fluctuation $O(N_{\mtc}^{-1/2})$, and since the coefficients $\alpha_\nu$ are bounded by assumption (following Remark~\ref{rem:bounded_kernel_weights}), the total MC fluctuation of the predictor scales as $O(\sqrt{P/N_{\mtc}})$.
\end{proof}

%%%%%%%%%%%%%%%%%%%%%%%%%%%%%%%%
%%%%%%%%%%%%%%%%%%%%%%%%%%%%%%%%
%%%%%%%%%%%%%%%%%%%%%%%%%%%%%%%%

\subsection{Propagating Through the Transformer}
\label{app:transformer_propagation}

In Proposition~\ref{prop:BlockPropagationComplexity}, we established that a single MC estimation of $\Cov(X_1,X_2)$ or $\Theta(X_1,X_2)$  yields an estimation error bounded by $O(N_{\mtc}^{-1/2})$ irrespective of the sequence length $T$. However, evaluating the kernel of a deep transformer involves a sequential chain of estimations, where the empirically evaluated covariance matrix of one block is passed into the non-linear Gaussian expectations of the next. A critical theoretical concern is whether the accumulation of MC errors across transformer blocks magnifies extensively with the sequence length $T$. For if it does, the number of MC samples, $N_{\mtc}$, will grow polynomially or above with $T$ in order to maintain a target accuracy at network output, thereby invalidating the inference-complexity bound established in Proposition~\ref{prop:BlockPropagationComplexity}.
Here, we show that this fortunately does not occur; the recursion maps for both the attention and MLP layers are Lipschitz continuous with respect to matrix max-norm, and, crucially, their Lipschitz constants are strictly $O(1)$. 

As in App.~\ref{app:transformer_block_propagation}, we let $\Cov'(X_1,X_2)$ and $\Theta'(X_1,X_2)$ denote the true input covariance matrix and NTK for a fixed pair of combinatorial task instance inputs $X_1,X_2 \in \bR^{T \times d}$ of a given transformer block, while 
\begin{align}
    \Cov'(X_1,X_2) \mapsto \Cov'(X_1,X_2) + \Delta\Cov'(X_1,X_2) \label{eq:CovShift}\\
    \Theta'(X_1,X_2) \mapsto \Theta'(X_1,X_2) + \Delta\Theta'(X_1,X_2) \label{eq:ThetaShift}
\end{align}
will be their perturbed MC estimates, with a sup-norm estimation error of
\begin{align}
    \norm{\Delta\Cov'(X_1,X_2)}_{\max} = \max_{a,b} \abs{\Delta\Cov'_{ab}(X_1,X_2)},\\
    \norm{\Delta\Theta'(X_1,X_2)}_{\max} = \max_{a,b} \abs{\Delta\Theta'_{ab}(X_1,X_2)}.
\end{align}
Following Eqs.~\eqref{eq:input_covariance_bounded} and \eqref{eq:input_ntk_bounded},
\begin{align}
    \norm{\Cov'(X_1,X_2)}_{\max} = O(1) = \norm{\Theta'(X_1,X_2)}_{\max}.
\end{align}

Let us fix a pair of tokens $a, b \in [T]$ and inputs $X_1, X_2 \in \bR^{T\times d}$, and denote
\begin{align}
\label{eq:s_j_notation}
    \bm{s}_1 = S_{a\bullet}^{h=1}(X_1), \quad \bm{s}_2 = S_{b\bullet}^{h=1}(X_2), \quad \bm{a}(\bm{s}_i) = \Softmax\left(\bm{s}_i\right).
\end{align}
Both $\bm{s}_i$ are considered random vectors in $\bR^T$. Taking the expectation over $\frac{1}{H} \sum_{h=1}^H$ implies that we may choose a single head to evaluate on, here, chosen arbitrarily to be $h=1$. 

The $2T$-dimensional random vector $\tilde{\bm{s}}=(\bm{s}_1,\bm{s}_2)$ follows a centered multivariate Gaussian distribution $\cN(0, C)$ with covariance matrix $C \in \bR^{2T \times 2T}$. A perturbation introduced in the previous block, as in Eqs.~\eqref{eq:CovShift} and \eqref{eq:ThetaShift}, also shifts  $C$ to $C + \Delta C$. As we have suppressed $X_1,X_2$ in our notation for $\bm{s}_i$, $i=1,2$, we will do so for $\Cov' \equiv \Cov'(X_1,X_2)$ and $\Theta' \equiv \Theta'(X_1,X_2)$ for the rest of this section. 

\begin{lemma}[$\Delta C$ bound] 
\label{lem:DeltaC}
$\norm{\Delta C}_{\max} \le 2\norm{\Cov'}_{\max} \norm{\Delta\Cov'}_{\max} + \norm{\Delta\Cov'}_{\max}^2$.
\end{lemma}

\begin{proof}
As established in Proposition~\eqref{prop:CovarianceFactorization}, the covariance $C \in \bR^{2T \times 2T}$ factorizes as 
\begin{align}
    C_{(i,c), (j,e)} = (\Cov_{ij}')_{ab} (\Cov_{ij}')_{ce},
\end{align}
where $\Cov'_{ij} \equiv \Cov'(X_i,X_j)$. Under perturbation, the corresponding score covariance is shifted
\begin{align}
    C_{(i,c), (j,e)} \mapsto \left((\Cov_{ij}')_{ab}+
    (\Delta\Cov_{ij}')_{ab}\right) \left((\Cov_{ij}')_{ce}+(\Delta\Cov_{ij}')_{ce}\right),
\end{align}
for $i,j \in \{1,2\}$ and $c,e \in [T]$.
The difference in the score covariance matrices is then
\begin{align}
    \Delta C_{(i,c), (j,e)} = (\Cov_{ij}')_{ab} (\Delta\Cov_{ij}')_{ce} + (\Delta\Cov_{ij}')_{ab} (\Cov_{ij}')_{ce} + (\Delta\Cov_{ij}')_{ab} (\Delta\Cov_{ij}')_{ce}
\end{align}
and therefore $\norm{\Delta C}_{\max} \le 2\norm{\Cov'}_{\max} \norm{\Delta\Cov'}_{\max} + \norm{\Delta\Cov'}_{\max}^2$.
\end{proof}

We will use the constraints set upon ${\bm a}({\bm s}) \in \bR^T$ by Softmax to establish useful bounds on its derivatives. 

\begin{lemma}[Derivative norms for Softmax] \label{lem:softmax_norms}
For $c\in [T]$, the following bounds hold
\begin{enumerate}[label=(\roman*)]
    \item $\norm{\nabla \bm{a}(\bm s)_c}_1 \leq 2\bm{a}(\bm s)_c$,
    \item $\norm{\nabla^2 \bm{a}(\bm s)_c}_{1,1} \leq 6\bm{a}(\bm s)_c$,
    \item $\norm{\nabla^3 \bm a(\bm s)_c}_{1,1,1} \le 26 \bm a(\bm s)_c$,
    \item $\norm{J_{\bm s} \bm{a}(\bm s)}_{1,1} \le 2$,
\end{enumerate}
where $\bm{a}(\bm s)$ is read as either $\bm{a}(\bm s_i)$, $i=1,2$.
\end{lemma}

\begin{proof}
(i) By \cite[Lemma 2]{Nair2026Softmax}, the first derivatives are given by $\partial_e \bm{a}(\bm s)_c = \bm{a}(\bm s)_c(\delta_{ce} - \bm{a}(\bm s)_e)$. Then
\begin{align}
    \norm{\nabla \bm{a}(\bm s)_c}_1 &= \sum_{e=1}^T \abs{\bm{a}(\bm s)_c(\delta_{ce} - \bm{a}(\bm s)_e)} \\&= \bm{a}(\bm s)_c(1-\bm{a}(\bm s)_c) + \sum_{e \neq c} \bm{a}(\bm s)_c \bm{a}(\bm s)_e \nonumber = 2\bm{a}(\bm s)_c(1-\bm{a}(\bm s)_c) \le 2\bm{a}(\bm s)_c.
\end{align}
(ii) The second derivatives are given by
\begin{align}
    \partial^2_{eb} \bm{a}(\bm s)_c &= \partial_b \left( \bm{a}(\bm s)_c \delta_{ce} - \bm{a}(\bm s)_c \bm{a}(\bm s)_e \right) \nonumber \\
    &= \bm{a}(\bm s)_c \delta_{ce} \delta_{cb} - \bm{a}(\bm s)_c \bm{a}(\bm s)_b \delta_{ce} \nonumber\\& - \bm{a}(\bm s)_c \bm{a}(\bm s)_e \delta_{cb} - \bm{a}(\bm s)_c \bm{a}(\bm s)_e \delta_{eb} + 2 \bm{a}(\bm s)_c \bm{a}(\bm s)_e \bm{a}(\bm s)_b.
\end{align}
Using $\sum_{e=1}^T \bm{a}(\bm s)_e = 1$ and $\bm{a}(\bm s)_e \ge 0$, the sum evaluates to
\begin{align}
    \sum_{e,b=1}^T \abs{\partial^2_{eb} \bm{a}(\bm s)_c} \le  \bm{a}(\bm s)_c + \bm{a}(\bm s)_c \cdot 1 + \bm{a}(\bm s)_c \cdot 1 + \bm{a}(\bm s)_c \cdot 1 + 2 \bm{a}(\bm s)_c \cdot 1 = 6 \bm{a}(\bm s)_c,
\end{align}
which proves the bound for $\norm{\nabla^2 \bm{a}(\bm s)_c}_{1,1}$. (iii) For the third derivatives, write
\begin{align}
    L_{ce} = \delta_{ce}-\bm a(\bm s)_e,
\end{align}
so that $\partial_e \bm a(\bm s)_c=\bm a(\bm s)_c L_{ce}$ and $\sum_{e=1}^T \abs{L_{ce}}\le 2$. From the second-derivative formula,
\begin{align}
    \partial^2_{eb}\bm a(\bm s)_c
    =
    \bm a(\bm s)_c
    \left(
        L_{ce}L_{cb}
        -
        \bm a(\bm s)_e L_{eb}
    \right).
\end{align}
Differentiating once more gives
\begin{multline}
    \partial^3_{ebr}\bm a(\bm s)_c
    =
    \bm a(\bm s)_c
    \Big[
        L_{cr}\left(L_{ce}L_{cb}-\bm a(\bm s)_e L_{eb}\right)
        -\bm a(\bm s)_eL_{er}L_{cb}
        -\bm a(\bm s)_bL_{br}L_{ce} \\
        -\bm a(\bm s)_eL_{er}L_{eb}
        +\bm a(\bm s)_e\bm a(\bm s)_bL_{br}
    \Big].
\end{multline}
Therefore,
\begin{align}
    \sum_{e,b,r=1}^T \abs{\partial^3_{ebr}\bm a(\bm s)_c}
    &\le
    \bm a(\bm s)_c
    \left[
        12+4+4+4+2
    \right]
    =
    26\bm a(\bm s)_c.
\end{align}
Indeed, the first contribution is bounded by
\begin{align}
    \sum_{e,b,r=1}^T
    \abs{L_{cr}}
    \abs{L_{ce}L_{cb}-\bm a(\bm s)_eL_{eb}}
    \le
    2\cdot 6,
\end{align}
using the second-derivative bound above. The remaining four contributions are bounded by
\begin{align}
    \sum_{e,b,r=1}^T \bm a(\bm s)_e\abs{L_{er}}\abs{L_{cb}}\le 4,\quad
    \sum_{e,b,r=1}^T \bm a(\bm s)_b\abs{L_{br}}\abs{L_{ce}}\le 4,
\end{align}
\begin{align}
    \sum_{e,b,r=1}^T \bm a(\bm s)_e\abs{L_{er}}\abs{L_{eb}}\le 4,\quad
    \sum_{e,b,r=1}^T \bm a(\bm s)_e\bm a(\bm s)_b\abs{L_{br}}\le 2.
\end{align}
This proves the bound for $\norm{\nabla^3\bm a(\bm s)_c}_{1,1,1}$. (iv) For the final bound,
\begingroup
\allowdisplaybreaks
\begin{align}
    \norm{J_{\bm s} \bm{a}(\bm s)}_{1,1} &= \sum_{c,e \in [T]} \abs{J_{\bm s} \bm{a}(\bm s)_{ce}} \nonumber\\
    &= \sum_{c \in [T]} \bm{a}(\bm s)_{c}(1-\bm{a}(\bm s)_{c}) + \sum_{c \neq e} \bm{a}(\bm s)_{c}\bm{a}(\bm s)_{e} \nonumber\\
    &= \sum_{c \in [T]} \bm{a}(\bm s)_{c}(1-\bm{a}(\bm s)_{c}) + \sum_{c, e \in [T]} \bm{a}(\bm s)_{c}\bm{a}(\bm s)_{e} - \sum_{c \in [T]} \bm{a}(\bm s)_{c}^2\nonumber\\
    &= \sum_{c \in [T]} \bm{a}(\bm s)_{c} - \sum_{c \in [T]} \bm{a}(\bm s)_{c}^2 + \left(\sum_{c \in [T]} \bm{a}(\bm s)_{c}\right)^2 - \sum_{c \in [T]} \bm{a}(\bm s)_{c}^2\nonumber\\
    &= (1-\norm{\bm{a}(\bm s)}^2_2) + (1-\norm{\bm{a}(\bm s)}^2_2)\nonumber\\
    &\le 2,
\end{align}
\endgroup
using \cite[Lemma 2]{Nair2026Softmax} to establish the second equality and the fact that $\bm{a}(\bm s)$ is a probability vector.
\end{proof}

\subsubsection{NNGP}

We shall begin with the token-to-token covariance matrix. By Proposition~\ref{prop:AttentionCovarianceUpdate}, the exact covariance update is given by
\begin{align}
    \Cov_{ab} = \bE_{\tilde{\bm{s}}} \left[ q_1(\tilde{\bm{s}}) \right], \quad
    q_1(\tilde{\bm{s}}) = \sum_{c,e=1}^T \bm{a}(\bm{s}_1)_c \Cov'_{ce} \bm{a}(\bm{s}_2)_e,
\end{align}
where $\bm{a}(\bm{s}) = \Softmax\left(\bm{s}\right)$. We can then bound its total discrepancy,
\begingroup
\allowdisplaybreaks
\begin{align}
\label{eq:DeltaCov}
    \abs{\Delta\Cov_{ab}} &=\abs{\bE_{\tilde{\bm{s}}\sim\cN(0,C+\Delta C)} \left[ \sum_{c,e=1}^T \bm{a}(\bm{s}_1)_c \left(\Cov'_{ce}+\Delta\Cov'_{ce}\right) \bm{a}(\bm{s}_2)_e \right] - \Cov_{ab}} \nonumber\\
    &\le \underbrace{\abs{\bE_{\tilde{\bm{s}}\sim\cN(0,C+\Delta C)} \left[ \sum_{c,e=1}^T \bm{a}(\bm{s}_1)_c \Cov'_{ce} \bm{a}(\bm{s}_2)_e \right] - \Cov_{ab}}}_{e_{\mathrm{measure}}^{\Cov}} \nonumber\\
    &+ \underbrace{\abs{\bE_{\tilde{\bm{s}}\sim\cN(0,C+\Delta C)} \left[\sum_{c,e=1}^T \bm{a}(\bm{s}_1)_c \Delta\Cov'_{ce} \bm{a}(\bm{s}_2)_e \right]}}_{e_{\mathrm{val}}^{\Cov}},
\end{align}
\endgroup
where the first error contribution, $e_{\mathrm{measure}}^{\Cov}$, comes from the indirect perturbative effect via $\bm{a}(\bm{s}_1)_c$ and $\bm{a}(\bm{s}_2)_e$, while the second, $e_{\mathrm{val}}^{\Cov}$, comes from the direct perturbative effect on $\Cov'_{ce}$.

\begin{lemma}[Bounding $e_{\mathrm{val}}^{\Cov}$]
\label{lem:e_val_bound}
    $e_{\mathrm{val}}^{\Cov} \leq \norm{\Delta\Sigma'}_{\max}$.
\end{lemma}

\begin{proof}
The error induced by substituting $\Cov'$ for $\Delta\Cov'$ can be bounded with
\begingroup
\allowdisplaybreaks
\begin{align}
    e_{\mathrm{val}}^{\Cov} &= \abs{ \bE_{\tilde{\bm{s}} \sim \cN(0,C+\Delta C)} \left[ \sum_{c,e=1}^T \bm{a}(\bm{s}_1)_c \Delta\Cov'_{ce} \bm{a}(\bm{s}_2)_e \right] } \nonumber \\ 
    &\le \norm{\Delta\Cov'}_{\max} \bE_{\tilde{\bm{s}} \sim \cN(0,C+\Delta C)} \left[ \left(\sum_{c=1}^T \bm{a}(\bm{s}_1)_c\right) \left(\sum_{e=1}^T \bm{a}(\bm{s}_2)_e\right) \right] \nonumber\\
    &\le \norm{\Delta\Cov'}_{\max}.    
\end{align}
\endgroup
The 2nd inequality results from $\Softmax$ outputs summing to $1$. The valuation error, thus, absorbs the $T^2$ summands and is exactly $1$-Lipschitz and independent of $T$.
\end{proof}

\begin{lemma}[Bounding $e_{\mathrm{measure}}^{\Cov}$] 
\label{lem:e_measure_bound}
$e_{\mathrm{measure}}^{\Cov} \le \frac{1}{2} \norm{\Delta C}_{\max} \sup_{\tilde{\bm{s}} \in \bR^{2T}} \norm{\nabla^2q(\tilde{\bm s})}_{1,1}$.
\end{lemma}

\begin{proof}
Here, we wish to bound the shift in the expectation due to the perturbed distribution of $\tilde{\bm{s}}$,
\begin{align}
    & e_{\mathrm{measure}}^{\Cov} = \nonumber\\ & \abs{\bE_{\tilde{\bm{s}}\sim\cN(0,C+\Delta C)} \left[ \sum_{c,e=1}^T \bm{a}(\bm{s}_1)_c \Cov'_{ce} \bm{a}(\bm{s}_2)_e \right] - \bE_{\tilde{\bm{s}}\sim\cN(0,C)} \left[ \sum_{c,e=1}^T \bm{a}(\bm{s}_1)_c \Cov'_{ce} \bm{a}(\bm{s}_2)_e \right]}.
\end{align}

To bound the shift in expectation under the perturbed Gaussian measure, we employ a multivariate Gaussian interpolation technique -- a standard application of Gaussian integration by parts, often associated with Price's Theorem. We construct a continuous interpolation between the unperturbed covariance $C$ and the perturbed covariance $C+\Delta C$ parameterized by $t \in [0, 1]$:
\begin{align}
    C(t) := C + t\Delta C.
\end{align}
with $\tilde{\bm{s}}(t) \sim \cN(0,C(t))$ interpolating accordingly. 
We define the expected value at $t$ to be
\begin{align}
    \varphi(t) = \bE_{\bm{s} \sim \tilde{\bm{s}}(t)} \left[ q_1(\tilde{\bm{s}}) \right].
\end{align}
The total error introduced by the measure pathway is the absolute difference at the endpoints, $e_{\mathrm{measure}}^{\Cov} = \abs{\varphi(1) - \varphi(0)}$. By the fundamental theorem of calculus, this is bounded by
\begin{align}
    e_{\mathrm{measure}}^{\Cov} = \abs{ \int_0^1 \varphi'(t) dt } \le \int_0^1 \abs{\varphi'(t)} dt.
\end{align}
To compute $\varphi'(t)$, we utilize Price's Theorem, 
\begin{align}
    \varphi'(t) 
    &= \frac{1}{2} \sum_{(i,c),(j,e)} \Delta C_{(i,c),(j,e)} \bE_{\tilde{\bm{s}}\sim\tilde{\bm{s}}(t)} \left[ \frac{\partial^2 q_1(\tilde{\bm{s}})}{\partial {\bm s}_{i,c} \partial {\bm s}_{j,e}} \right].
\end{align}
Substituting the derivative back into the integral, 
\begin{align}
    e_{\mathrm{measure}}^{\Cov} &\le \frac{1}{2} \int_0^1 \sum_{(i,c),(j,e)} \abs{\Delta C_{(i,c),(j,e)}} \, \abs{ \bE_{\tilde{\bm{s}}\sim\tilde{\bm{s}}(t)} \left[ \frac{\partial^2 q_1(\tilde{\bm{s}})}{\partial {\bm s}_{i,c} \partial {\bm s}_{j,e}} \right]} dt \nonumber \\
    &\le \frac{1}{2} \norm{\Delta C}_{\max} \int_0^1 \bE_{\tilde{\bm{s}}\sim\tilde{\bm{s}}(t)} \left[ \sum_{(i,c),(j,e)} \abs{  \frac{\partial^2 q_1(\tilde{\bm{s}})}{\partial {\bm s}_{i,c} \partial {\bm s}_{j,e}} } \right] dt.
\end{align}
We can now obtain a strict, distribution-free upper bound by replacing the expression inside the expectation with its supremum. Because this supremum is independent of the interpolation parameter $t$, both expectation and integral trivially evaluate to 1, giving us
\begin{align} 
    e_{\mathrm{measure}}^{\Cov} \le \frac{1}{2} \norm{\Delta C}_{\max} \sup_{\tilde{\bm{s}} \in \bR^{2T}} \left\{\sum_{(i,c),(j,e)} \abs{  \frac{\partial^2 q_1(\tilde{\bm{s}})}{\partial {\bm s}_{i,c} \partial {\bm s}_{j,e}} } \right\}.
\end{align}
Finally, we identify
\begin{align}
\label{eq:hessian_l1_norm}
    \norm{\nabla^2q_1(\tilde{\bm s})}_{1,1} = \sum_{(i,c),(j,e)} \abs{  \frac{\partial^2 q_1(\tilde{\bm{s}})}{\partial {\bm s}_{i,c} \partial {\bm s}_{j,e}} },
\end{align}
which completes our proof.
\end{proof}

Next, we bound $\norm{\nabla^2q(\tilde{\bm s})}_{1,1}$ in Lemma~\ref{lem:e_measure_bound} by showing that each summand of Eq.~\eqref{eq:hessian_l1_norm} scales as $O(1/T^2)$. To this end, we use the constraints set upon $\bm a(\bm s) \in \bR^T$ by Softmax, as established in Lemma~\ref{lem:softmax_norms}.

\begin{lemma}[$L_1$-bounded Hessian of $q_1$] \label{lem:bounded_hessian}
$\sup_{\tilde{\bm{s}} \in \bR^{2T}} \norm{\nabla^2q_1(\tilde{\bm s})}_{1,1} \le 20\norm{\Cov'}_{\max}$.
\end{lemma}

\begin{proof}
By definition,
\begin{align}
\label{eq:hessian_expansion}
    \norm{\nabla^2 q_1(\tilde{\bm s})}_{1,1} = \sum_{i \in \{1,2\}} \sum_{c,e=1}^T \abs{\frac{\partial^2 q_1}{\partial {\bm s}_{i,c} \partial {\bm s}_{i,e}}} + 2 \sum_{c,e=1}^T \abs{\frac{\partial^2 q_1}{\partial {\bm s}_{1,c} \partial {\bm s}_{2,e}}}. 
\end{align}
For the first summand on the right-hand side, with $i=1$,
\begingroup
\allowdisplaybreaks
\begin{align}
    \sum_{c,e=1}^T \abs{\frac{\partial^2 q_1}{\partial {\bm s}_{1,c} \partial {\bm s}_{1,e}}} & =
    \sum_{c,e=1}^T \abs{ \sum_{g,h=1}^T \frac{\partial^2 \bm{a}(\bm s_1)_g}{\partial (\bm s_1)_c \partial (\bm s_1)_e} \Cov'_{gh} \bm{a}(\bm s_2)_h} \nonumber \\
    & \le \sum_{g,h=1}^T \left(\sum_{c,e=1}^T \abs{\frac{\partial^2 \bm{a}(\bm s_1)_g}{\partial (\bm s_1)_c \partial (\bm s_1)_e}} \right) \abs{\Cov'_{gh} \bm{a}(\bm s_2)_h}  \nonumber \\    
    & = \sum_{g,h=1}^T \norm{\nabla^2 \bm{a}(\bm s_1)_g}_{1,1} \abs{\Cov'_{gh} \bm{a}(\bm s_2)_h} \nonumber \\
    & \le 6 \norm{\Cov'}_{\max} \left(\sum_{g=1}^T \bm{a}(\bm s_1)_g\right) \left(\sum_{h=1}^T \bm{a}(\bm s_2)_h\right) \nonumber \\
    & \le 6 \norm{\Cov'}_{\max},
\end{align}
\endgroup
applying Lemma~\ref{lem:softmax_norms}(ii) and the row-stochasticity of the attention matrix in the last two inequalities above. 
The same applies, of course, to the case $i=2$. For the second summand on the right-hand side of Eq.~\eqref{eq:hessian_expansion}, we have
\begingroup
\allowdisplaybreaks
\begin{align}
    \sum_{c,e=1}^T \abs{\frac{\partial^2 q_1}{\partial {\bm s}_{1,c} \partial {\bm s}_{2,e}}} & = \sum_{c,e=1}^T \abs{\sum_{g,h=1}^T  \frac{\partial \bm{a}(\bm s_1)_g}{\partial {\bm s}_{1,c}} \Cov'_{gh} \frac{\partial \bm{a}(\bm s_1)_h}{\partial {\bm s}_{2,e}}} \nonumber \\
    &\le \sum_{g,h=1}^T \left(\sum_{c=1}^T \abs{\frac{\partial \bm{a}(\bm s_1)_g}{\partial {\bm s}_{1,c}}}\right) \abs{\Cov'_{gh}} \left(\sum_{e=1}^T \abs{\frac{\partial \bm{a}(\bm s_2)_h}{\partial {\bm s}_{2,e}}}\right) \\
    &\le \norm{\Cov'}_{\max} \sum_{g,h=1}^T \norm{\nabla \bm{a}(\bm s_1)_g}_1 \norm{\nabla \bm{a}(\bm s_2)_h}_1 \nonumber \\
    &\le 4\norm{\Cov'}_{\max} \left(\sum_{g=1}^T \bm{a}(\bm s_1)_g\right) \left(\sum_{h=1}^T \bm{a}(\bm s_2)_h\right) \nonumber \\
    & \le 4\norm{\Cov'}_{\max}.
\end{align}
\endgroup
applying Lemma~\ref{lem:softmax_norms}(i) and the row-stochasticity of the attention matrix again. Substituting these inequalities into Eq.~\eqref{eq:hessian_expansion}, we get
\begin{align}
    \sup_{\tilde{\bm{s}} \in \bR^{2T}} \norm{\nabla^2q_1(\tilde{\bm s})}_{1,1} \le 20\norm{\Cov'}_{\max},
\end{align}
which completes the proof.
\end{proof}

\subsubsection{NTK}

We move on to discussing the discrepancy in the NTK kernel evaluation. By the recursion relation in Eq.~\eqref{eq:ntk_recursion}, we need only consider the trace terms. Let us denote
\begin{align}
    F = \bE_{\tilde{\bm s}\sim\cN(0,C)}[q_2(\tilde{\bm s})], \quad q_2(\tilde{\bm s}) = \mathrm{Tr}\left(
            \left(J_{\bm s_1}\bm a(\bm s_1)\right)^\top
            \Cov'
            J_{\bm s_2}\bm a(\bm s_2)
            (\Cov')^\top
        \right), \label{eq:trace_term_ntk}\\
    G = \bE_{\tilde{\bm s}\sim\cN(0,C)}[q_3(\tilde{\bm s})], \quad q_3(\tilde{\bm s}) = \mathrm{Tr}\left(
            \left(J_{\bm s_1}\bm a(\bm s_1)\right)^\top
            \Cov'
            J_{\bm s_2}\bm a(\bm s_2)
            (\Theta')^\top \label{eq:trace_term_ntk_2nd}
        \right).
\end{align}
We will focus on the discrepancy in $F$ and the discrepancy in $G$ will follow in similar fashion. The former is given by
\begin{align}
\label{eq:trace_discrepancy_ntk}
    \abs{\Delta F}
    & = \abs{\bE_{\tilde{\bm s}\sim\cN(0,C+\Delta C)}
    \left[
        \mathrm{Tr}\left(
            \left(J_{\bm s_1}\bm a(\bm s_1)\right)^\top
            (\Cov' + \Delta\Cov')
            J_{\bm s_2}\bm a(\bm s_2)
            (\Cov' + \Delta\Cov')^\top
        \right)
    \right]
    - F} \nonumber\\
    & \le e_{\mathrm{val}}^{F}+e_{\mathrm{measure}}^{F}
\end{align}
with value and measure errors given by
\begin{multline}
\label{eq:e_val_ntk}
    e_{\mathrm{val}}^{F}
    =
    \left|
    \bE_{\tilde{\bm s}\sim\cN(0,C+\Delta C)}
    \left[
        \mathrm{Tr}\left(
            \left(J_{\bm s_1}\bm a(\bm s_1)\right)^\top
            (\Cov' + \Delta\Cov')
            J_{\bm s_2}\bm a(\bm s_2)
            (\Cov' + \Delta\Cov')^\top
        \right)
        - \right.\right.\\
    \left.\left.
        \mathrm{Tr}\left(
            \left(J_{\bm s_1}\bm a(\bm s_1)\right)^\top
            \Cov'
            J_{\bm s_2}\bm a(\bm s_2)
            (\Cov')^\top
        \right)
    \right]
    \right|
\end{multline}
and
\begin{align}
\label{eq:e_measure_ntk}
    e_{\mathrm{measure}}^{F}
    =
    \left|
    \bE_{\tilde{\bm s}\sim\cN(0,C+\Delta C)}
    \left[
        \mathrm{Tr}\left(
            \left(J_{\bm s_1}\bm a(\bm s_1)\right)^\top
            \Cov'
            J_{\bm s_2}\bm a(\bm s_2)
            (\Cov')^\top
        \right)
    \right]
    -
    F
    \right|
\end{align}

\begin{lemma}[Bounding $e_{\mathrm{val}}^{F}$]
\label{lem:e_val_bound_ntk}
    $e_{\mathrm{val}}^{F} \le 8\norm{\Delta\Cov'}_{\max}\norm{\Cov'}_{\max} + 4\norm{\Delta\Cov'}_{\max}^2$.
\end{lemma}

\begin{proof}
We begin by identifying three terms in Eq.~\eqref{eq:e_val_ntk}.
\begin{align}
    e_{\mathrm{val}}^\Theta \le e_1 + e_2 + e_3,
\end{align}
with
\begin{align}
    e_1 &= \abs{\bE_{\tilde{\bm s}\sim\cN(0,C+\Delta C)}\left[\mathrm{Tr}\left(\left(J_{\bm s_1}\bm a(\bm s_1)\right)^\top \Delta\Cov' J_{\bm s_2}\bm a(\bm s_2)(\Cov')^\top\right)\right]},\\
    e_2 &= \abs{\bE_{\tilde{\bm s}\sim\cN(0,C+\Delta C)}\left[\mathrm{Tr}\left(\left(J_{\bm s_1}\bm a(\bm s_1)\right)^\top \Cov' J_{\bm s_2}\bm a(\bm s_2)(\Delta\Cov')^\top\right)\right]},\\    
    e_3 &= \abs{\bE_{\tilde{\bm s}\sim\cN(0,C+\Delta C)}\left[\mathrm{Tr}\left(\left(J_{\bm s_1}\bm a(\bm s_1)\right)^\top \Delta\Cov' J_{\bm s_2}\bm a(\bm s_2)(\Delta\Cov')^\top\right)\right]}.    
\end{align}
Repeating the calculation of Eq.~\eqref{eq:trace_bound}, for each term, we get 
\begin{align}
    e_1+e_2+e_3 \le 8\norm{\Delta\Cov'}_{\max}\norm{\Cov'}_{\max} + 4\norm{\Delta\Cov'}_{\max}^2,
\end{align}
as stated, utilizing Lemma~\ref{lem:softmax_norms}(iv).
\end{proof}

\begin{lemma}[Bounding $e_{\mathrm{measure}}^{F}$]
\label{lem:e_measure_bound_ntk}
    $e_{\mathrm{measure}}^{\Theta} \le \frac{1}{2}\norm{\Delta C}_{\max}\sup_{\tilde{\bm s}}\norm{\nabla_{\tilde{\bm s}}^2 q_2(\tilde{\bm s})}_{1,1}$.
\end{lemma}

\begin{proof}
The result follows from the exact same Gaussian interpolation argument as in Lemma~\ref{lem:e_measure_bound}, replacing $q_1$ with $q_2$.
\end{proof}

\begin{lemma}[$L_1$-bounded Hessian of $q_2$] \label{lem:bounded_hessian_ntk}
$\sup_{\tilde{\bm s}\in\bR^{2T}}\norm{\nabla^2 q_2(\tilde{\bm s})}_{1,1} \le 176\norm{\Cov'}_{\max}^2$.
\end{lemma}

\begin{proof}
Expanding $q_2$ of Eq.~\ref{eq:trace_term_ntk} entrywise,
\begin{align}
    q_2(\tilde{\bm s}) = \sum_{g,h,r,m=1}^T \Cov'_{gh}\Cov'_{rm}\frac{\partial \bm a(\bm s_1)_g}{\partial \bm s_{1,r}}\frac{\partial \bm a(\bm s_2)_h}{\partial \bm s_{2,m}}.
\end{align}
As in Eq.~\eqref{eq:hessian_expansion},
\begin{align}
    \norm{\nabla^2 q_2(\tilde{\bm s})}_{1,1} = \sum_{i \in \{1,2\}} \sum_{c,e=1}^T \abs{\frac{\partial^2 q_2}{\partial {\bm s}_{i,c} \partial {\bm s}_{i,e}}} + 2 \sum_{c,e=1}^T \abs{\frac{\partial^2 q_2}{\partial {\bm s}_{1,c} \partial {\bm s}_{2,e}}}.
\end{align}
For the first summand, with $i=1$,
\begin{align}
    \sum_{c,e=1}^T \abs{\frac{\partial^2 q_2}{\partial {\bm s}_{1,c}\partial {\bm s}_{1,e}}}
    &\le \sum_{g,h,r,m=1}^T \abs{\Cov'_{gh}\Cov'_{rm}}\left(\sum_{c,e=1}^T\abs{\frac{\partial^3 \bm a(\bm s_1)_g}{\partial \bm s_{1,r} \partial \bm s_{1,c} \partial \bm s_{1,e}}} \right)\abs{\frac{\partial \bm a(\bm s_2)_h}{\partial \bm s_{2,m}}} \nonumber\\
    &\le \norm{\Cov'}_{\max}^2 \left(\sum_{g=1}^T \norm{\nabla^3 \bm a(\bm s_1)_g}_{1,1,1}\right)\left(\sum_{h=1}^T \norm{\nabla \bm a(\bm s_2)_h}_1\right) \nonumber\\
    &\le 52\norm{\Cov'}_{\max}^2.
\end{align}
using Lemma~\ref{lem:softmax_norms}(i) and (iii). The same bound holds for $i=2$. For the mixed summand,
\begin{align}
    \sum_{c,e=1}^T \abs{\frac{\partial^2 q_2}{\partial {\bm s}_{1,c}\partial {\bm s}_{2,e}}}
    &\le \sum_{g,h,r,m=1}^T \abs{\Cov'_{gh}\Cov'_{rm}}\left(\sum_{c=1}^T\abs{\frac{\partial^2 \bm a(\bm s_1)_g}{\partial \bm s_{1,r} \partial \bm s_{1,c}}}\right)\left(\sum_{e=1}^T\abs{\frac{\partial^2 \bm a(\bm s_2)_h}{\partial \bm s_{2,m} \partial \bm s_{2,e}}}\right) \nonumber\\
    &\le \norm{\Cov'}_{\max}^2\left(\sum_{g=1}^T\norm{\nabla^2 \bm a(\bm s_1)_g}_{1,1}\right)\left(\sum_{h=1}^T\norm{\nabla^2 \bm a(\bm s_2)_h}_{1,1}\right) \nonumber\\
    &\le 36\norm{\Cov'}_{\max}^2.
\end{align}
using Lemma~\ref{lem:softmax_norms}(ii). Substituting these inequalities into the Hessian expansion gives
\begin{align}
    \norm{\nabla^2 q_2(\tilde{\bm s})}_{1,1} \le 52\norm{\Cov'}_{\max}^2 + 52\norm{\Cov'}_{\max}^2 + 2\cdot 36\norm{\Cov'}_{\max}^2 = 176\norm{\Cov'}_{\max}^2.
\end{align}
Taking the supremum over $\tilde{\bm s}\in\bR^{2T}$ completes the proof.
\end{proof}

\subsubsection{Theorem~\ref{thm:TransformerPropagationComplexity}}

\TransformerPropagationComplexity*

\begin{proof}
Applying lemmas~\ref{lem:e_val_bound}, \ref{lem:e_measure_bound}, and \ref{lem:bounded_hessian} to Eq.~\eqref{eq:DeltaCov}, we get 
\begin{align}
    \abs{\Delta\Cov_{ab}} \le \norm{\Delta\Cov'}_{\max} + 10 \norm{\Delta C}_{\max} \norm{\Cov'}_{\max},
\end{align}
and, when we apply Lemma~\ref{lem:DeltaC}, we get 
\begin{align}
    \abs{\Delta\Cov_{ab}} \le \norm{\Delta\Cov'}_{\max} + 10 \left(2\norm{\Cov'}_{\max} \norm{\Delta\Cov'}_{\max} + \norm{\Delta\Cov'}_{\max}^2\right) \norm{\Cov'}_{\max}.
\end{align}
Thus,
\begin{align}
    \norm{\Delta\Cov}_{\max}
    \le
    \left(1+20\norm{\Cov'}_{\max}^2\right)\norm{\Delta\Cov'}_{\max}
    +
    10\norm{\Cov'}_{\max}\norm{\Delta\Cov'}_{\max}^2,
\end{align}
and propagating the covariance matrix through the attention layer is locally Lipschitz with a $T$-independent Lipschitz constant by Eq.~\eqref{eq:input_covariance_bounded}.

As for the NTK, by the recursion relation in Eq.~\eqref{eq:ntk_recursion}, the propagated term involving $A_1\Theta'A_2^\top$ is handled exactly as in the covariance case, with \(\Cov'\) replaced by \(\Theta'\) in the integrand, while the score-law perturbation still enters only through \(\Delta C\). The perturbation
\begin{align}
    \bE_{\tilde{\bm s} \sim \cN(0,C+\Delta C)} \left[\sum_{c,e=1}^T \bm{a}(\bm{s}_1)_c (\Theta'_{ce}+\Delta\Theta'_{ce}) \bm{a}(\bm{s}_2)_e\right]
\end{align}
can thus be handled $T$-independently, applying lemmas~\ref{lem:e_val_bound}, \ref{lem:e_measure_bound}, and \ref{lem:bounded_hessian}, as they only rely on the bounded max-norm of the input kernel. 

For the first trace term, $F$ (Eq.~\eqref{eq:trace_term_ntk}), applying lemmas~\ref{lem:e_val_bound_ntk}, \ref{lem:e_measure_bound_ntk}, and \ref{lem:bounded_hessian_ntk} together with Lemma~\ref{lem:DeltaC} provides us with a $T$-independent bound on the discrepancy $\abs{\Delta F}$ in Eq.~\eqref{eq:trace_discrepancy_ntk} with a Lipschitz constant solely dependent on $\norm{\Cov'}_{\max}$. The second trace term $G$ (Eq.~\eqref{eq:trace_term_ntk_2nd}) can be pursued in a similar fashion. Here, the Lipschitz constant will depend on both $\norm{\Cov'}_{\max}$ and $\norm{\Theta'}_{\max}$ and will be $T$-independent nonetheless. Now that we have bounded trace terms, we need to consider the discrepancies of their associate full expressions in Eq.~\eqref{eq:ntk_recursion}, namely,
\begin{align}
    \abs{\Delta\left((2\Cov'_{ab}+\Theta'_{ab})F\right)}
    &\le \abs{2\Delta\Cov'_{ab}+\Delta\Theta'_{ab}}\abs{F+\Delta F}
    + \abs{2\Cov'_{ab}+\Theta'_{ab}}\abs{\Delta F}\label{eq:full_F_perturbation}\\
    \abs{\Delta(\Cov'_{ab}G)}
    &\le \abs{\Delta\Cov'_{ab}}\abs{G+\Delta G}
    + \abs{\Cov'_{ab}}\abs{\Delta G}.\label{eq:full_G_perturbation}
\end{align}
Since $F$ and $G$ are uniformly bounded by the trace estimate coming from Eq.~\eqref{eq:trace_bound},
\begin{align}
    \abs{F} \le 4\norm{\Cov'}_{\max}^2, \quad \abs{G} \le 4\norm{\Cov'}_{\max}\norm{\Theta'}_{\max}.
\end{align}
Since $\Delta F$ and $\Delta G$ satisfy $T$-independent local Lipschitz estimates in $\norm{\Delta\Cov'}_{\max}$ and $\norm{\Delta\Theta'}_{\max}$, this gives a $T$-independent local Lipschitz constant in both perturbations in Eqs.~\eqref{eq:full_F_perturbation} and \eqref{eq:full_G_perturbation}.

The element-wise update rule for the covariance matrix through the MLP layer is given in Eq.~\eqref{eq:mlp_propagation}. Because there is no summation or mixing across the $T$ sequence indices, a perturbation $\norm{\Delta\Cov}_{\max}$ merely perturbs $k$ of Eq.~\eqref{eq:mlp_2x2_covariance} by at most $\norm{\Delta\Cov}_{\max}$. For standard activation functions bounded by linear growth, such as ReLU or GeLU, the dual activation integral is locally Lipschitz in the input covariance entries on the nondegenerate covariance domain considered here \citep{Schoenholz2016}. Consequently, the local error scales by a bounded and strictly architecture-dependent constant.
For the NTK update through the MLP, Eq.~\eqref{eq:mlp_ntk_propagation} takes the form
\begin{align}
    \Theta^{\mathrm{out}}_{ab}
    =
    \Cov^{\mathrm{out}}_{ab}
    +
    \Theta_{ab}\dot{\Cov}^{\mathrm{out}}_{ab},
\end{align}
where \(\dot{\Cov}^{\mathrm{out}}_{ab}\) denotes the normalized derivative dual kernel. Therefore,
\begin{align}
    \abs{\Delta\Theta^{\mathrm{out}}_{ab}}
    \le
    \abs{\Delta\Cov^{\mathrm{out}}_{ab}}
    +
    \abs{\Delta\Theta_{ab}}\abs{\dot{\Cov}^{\mathrm{out}}_{ab}+\Delta\dot{\Cov}^{\mathrm{out}}_{ab}}
    +
    \abs{\Theta_{ab}}\abs{\Delta\dot{\Cov}^{\mathrm{out}}_{ab}}.
\end{align}
The derivative dual kernel is bounded and locally Lipschitz in the $2\times 2$ covariance $k$, with $T$-independent constants. This again gives a $T$-independent local Lipschitz constant in both $\norm{\Delta\Cov'}_{\max}$ and $\norm{\Delta\Theta'}_{\max}$.

Since both the attention and MLP kernel recursions possess $T$-independent max-norm Lipschitz constants, an initial MC error will magnify by at most $(\ell_{\mathrm{Att}}\ell_{\mlp})^L$, for the Lipschitz constants $\ell_{\mathrm{Att}}$ and $\ell_{\mlp}$ of the attention and MLP layers (absorbing LayerNorm), across $L$ transformer blocks. For a  finite-depth network, this is a fixed constant factor. In particular, the number of MC samples, $N_{\mtc}$, may need to scale exponentially in $L$ in the worst case to maintain the same level of MC error, but this scaling is agnostic of $T$. 
\end{proof}

\section{Convergence to NNGP}
\label{app:perturbation_theory}

Here, we provide analytic support to Assumption~\ref{assm:discrepancy_scale}. This complements existing results \citep{huang2020dynamics,huang2021neuraltangentkerneldeep}, which assume a fixed input dimension. 
We proceed by adopting a Bayesian viewpoint and, for simplicity, focus on predicting the $T$-th token
\begin{align}
    f(X) \equiv \bm{z}^{\mathrm{out}}_T(X),\quad  X \in \bR^{T \times d}.
\end{align}
Given an input dataset $\cD = \{(X_\nu,y_\nu)\}_{\nu=1}^P$, we let $K$ denote the $T$-th output kernel matrix
\begin{align}   K_{\mu\nu}=K(X_\mu,X_\nu)\equiv\Sigma^{\scriptscriptstyle \mathrm{out}}_{TT}(X_\mu,X_\nu).
\end{align}
By \citet[Theorem 3]{ChoSaul2009}, in the infinite-width limit, $f(X) \sim \mathcal{GP}(0,K)$, which results in the NNGP kernel predictor, for a test point $X_* \in \bR^{T \times d}$, 
\begin{align}
    f_K(X_*) = \bm{k}^{*\top} \tilde{K}^{-1} \bm{y}
\end{align}
with a symmetric positive definite (PD)
\begin{align}
    \tilde{K}=K+\kappa I_P,
\end{align}
$\kappa>0$ being the ridge parameter (or observation noise), a vector of covariances
\begin{align}
    \bm{k}^* = (K(X_*,X_\nu))_{\nu \in [P]} \in \bR^P,
\end{align}
and $k^* = K(X_*,X_*) \in \bR$.

\begin{lemma}[Bounds for $\tilde{K}$]
\label{lem:tilde_K_bounds}
    Under the above setting,
    \begin{align}
        \norm{\tilde{K}^{-1}}_{\max} \leq \frac{1}{\kappa}, \qquad \norm{\tilde{K}^{-1}}_2 \leq \frac{1}{\kappa}, \qquad \norm{\tilde{K}^{-1}K\tilde{K}^{-1}}_{\max} \leq \frac{1}{4\kappa}.
    \end{align}
\end{lemma}

\begin{proof}
Consider the spectral norm of $\tilde{K}^{-1}$. By definition,
\begin{align}
    \|\tilde{K}^{-1}\|_2 = \sup\left\{\abs{\bm{x}_1^\top\tilde{K}^{-1}\bm{x}_2}|\bm{x}_i\in\bR^P: \norm{\bm{x}_1}=1=\norm{\bm{x}_2}\right\}.
\end{align}
Choosing $\bm{x}_1=\bm{e}_\mu$ and $\bm{x}_2=\bm e_\nu$, we get
\begin{align}
    \abs{\tilde{K}^{-1}_{\mu\nu}} \leq \norm{\tilde{K}^{-1}}_2.
\end{align}
Using the fact that ${K}$ is symmetric positive semi-definite (PSD) and $\tilde{K}$ is symmetric PD, we have 
\begin{align}
    \Spec\left(\tilde{K}^{-1}\right) = \left\{\frac{1}{\lambda+\kappa}\right|\left.\lambda\in\Spec(K)\vphantom{\frac{1}{\lambda+\kappa}}\right\},
\end{align}
so that
\begin{align}
    \abs{\tilde{K}^{-1}_{\mu\nu}} \leq \norm{\tilde{K}^{-1}}_2 = \max_{\lambda \in \Spec(K)} \Spec(\tilde{K}^{-1}) \leq \frac{1}{\kappa},
\end{align}
which establishes the two left inequalities. Similarly, we have
\begin{align}
    \Spec(\tilde{K}^{-1}K\tilde{K}^{-1}) = \left\{\frac{\lambda}{(\lambda+\kappa)^2}\right|\left.\lambda\in\Spec(K)\vphantom{\frac{\lambda}{(\lambda+\kappa)^2}}\right\}.
\end{align}
Thus,
\begin{align}
    \norm{\tilde{K}^{-1}K\tilde{K}^{-1}}_2 = \max_{\lambda \in \Spec(K)} \Spec(\tilde{K}^{-1}K\tilde{K}^{-1}) \leq \frac{1}{4\kappa},
\end{align}
which establishes the right inequality.
\end{proof}

\NNGPConvergenceBound*

\begin{proof}
\citet{NavehSompolinskyRingel} studied the discrepancy between a deep neural network (DNN), $f$, trained with full-batch gradient descent (GD), weight decay, and additive white Gaussian noise, and its NNGP kernel predictor ${f}_K$ for the same training set $\cD$. At a test point $X_*$, \citet{NavehSompolinskyRingel} have 
\begin{align}
    \bE_{f \sim P[f]}[{f}(X_*)] - f_{K}(X_*) = \frac{1}{N} f_U(X_*) + O\left(\frac{1}{N^2}\right),
\end{align}
where $P[f]$ is the posterior on the trained $f$. While this work focused on FCN and convolutional neural networks (CNNs), its results apply to transformers, provided we focus on single-token predictions and, similarly to all works that discuss kernel limits of classification tasks, replace the cross-entropy loss on the logits with an MSE-loss w.r.t one-hot encoding of the categorical labels \citep{Lee2018DNN}.  
Following this, one again has a Gaussian posterior at infinite width for any specific next-token logit and an Edgeworth expansion controlled by $1/N$ and the fourth cumulant for the leading non-Gaussian corrections. The results we cite below rely solely on this structure.\footnote{As in the original work, these results do not apply for mean-field scaling where the theory at infinite $N$ stops being Gaussian.} The leading order correction term, $f_U(X_*)$, is given in \citep[Eq.~(11)]{NavehSompolinskyRingel} by
\begin{align} 
\label{eq:FWC_main_result}
{f}_{U}(X_*) = 
\frac{1}{6} \sum_{\mu_1\mu_2\mu_3=1}^P \tilde{U}_{\mu_{1}\mu_{2}\mu_{3}}^*  \left(\tilde{y}_{\mu_{1}}\tilde{y}_{\mu_{2}}\tilde{y}_{\mu_{3}} - 3\tilde{K}_{\mu_{1}\mu_{2}}^{-1}\tilde{y}_{\mu_{3}} \right),
\end{align}
where $U$ is the $4$th-order functional cumulant of the output \citep[Eq.~(9)]{NavehSompolinskyRingel} and
\begin{align}
    \tilde{y}_{\mu_i} = \sum_{\nu=1}^P \tilde{K}_{\mu_i\nu}^{-1}y_{\nu}, \quad
\tilde{U}_{\mu_{1}\mu_{2}\mu_{3}}^* = U_{\mu_{1}\mu_{2}\mu_{3}}^{*} -  \sum_{\nu,\nu'=1}^P U_{\mu_{1}\mu_{2}\mu_{3}\nu}  \tilde{K}_{\nu\nu'}^{-1} \bm{k}^*_{\nu'}. \label{eq:U_tilde_star_def}
\end{align}
The asterisk on the right-hand side of Eq.~\eqref{eq:U_tilde_star_def} indicates evaluation, namely, $\bm{k}^*_{\nu'} = K(X_*,X_{\nu'})$ and
$U_{\mu_{1}\mu_{2}\mu_{3}}^{*}=U\left(X_{*},X_{\mu_{1}},X_{\mu_{2}},X_{\mu_{3}}\right)$ and $y_{\nu}$ is either a regression target or an indicator function for the logit taking values of $0,1$ (or with bias free one hot encoding, $-1/V,1-1/V$ where $V$ is the size of the vocabulary).  

The scale of ${f}_{U}(X_*)$ as a function of $P$ determines the scale of the discrepancy between the NNGP kernel predictor and the finite DNN. We will therefore bound the basic ingredients of  Eq.~\eqref{eq:FWC_main_result} independently of $T$. First, by definition, we have $y_{\mu}=O(1)$. We also have $\tilde{K}_{\mu\nu}^{-1}=O(1)$ by Lemma~\ref{lem:tilde_K_bounds}. Then,
\begin{align}
    \norm{\tilde{\bm y}}_2 \le \norm{\tilde{K}^{-1}}_2 \norm{\bm y}_2 \le \frac{\norm{\bm y}_2}{\kappa} = O(\sqrt{P}),
\end{align}
applying Lemma~\ref{lem:tilde_K_bounds} again. Second, the output $f(X)$ has an $O(1)$ Gaussian tail at initialization, and its variance is bounded irrespective of $T$, yielding $\bm{k}^*_{\mu} = O(1)$. Similarly, the $4$th-order cumulant between any four points, including a test point, also obeys $U_{\mu_1...\mu_4} =O(1)$ and $U^*_{\mu_1...\mu_3} = O(1)$. A crude estimate would, thus, yield $\tilde{U}_{\mu_{1}\mu_{2}\mu_{3}}^* = O(P^2)$. However, we will show  $\tilde{U}_{\mu_{1}\mu_{2}\mu_{3}}^* = O(\sqrt{P})$.

By definition, for a function $u$ in the reproducing kernel Hilbert
space (RKHS), $\cH_K$, of $K$, we have
\begin{align}
    u(X) = \langle K(X,\cdot),u\rangle_{{\cH}_K},
\end{align}
and, in particular, $\langle K(X,\cdot) , K(X',\cdot)\rangle_{{\cH}_K} = K(X,X')$. Consequently,
\begin{align}
\label{eq:rkhs_bound}
    \abs{u(X)} \leq \norm{u}_{{\cH}_K} \norm{K(X,\cdot)}_{{\cal H}_K}=\norm{u}_{{\cH}_K} \sqrt{K(X,X)}.
\end{align}
Next, taking
\begin{align}
    u_K(X) = \sum_{\nu,\nu'=1}^P U_{\mu_{1}\mu_{2}\mu_{3}\nu}  \tilde{K}_{\nu\nu'}^{-1} K(X,X_{\nu'}),
\end{align}
we note that $u_K(X_*) = U_{\mu_{1}\mu_{2}\mu_{3}}^{*} - \tilde{U}_{\mu_{1}\mu_{2}\mu_{3}}^*$. Its $\cH_K$-norm obeys
\begin{align}
    \norm{u_K}^2_{\cH_K} & = \left\langle \sum_{\mu,\mu'=1}^P U_{\mu_{1}\mu_{2}\mu_{3}\mu}  \tilde{K}_{\mu\mu'}^{-1} K(\cdot,X_{\mu'}), \sum_{\nu,\nu'=1}^P U_{\mu_{1}\mu_{2}\mu_{3}\nu}  \tilde{K}_{\nu\nu'}^{-1} K(\cdot,X_{\nu'}) \,\right\rangle_{\cH_K} \nonumber\\
    & = \sum_{\mu,\mu'=1}^P \sum_{\nu,\nu'=1}^P U_{\mu_{1}\mu_{2}\mu_{3}\mu}  \tilde{K}_{\mu\mu'}^{-1} \langle K(\cdot,X_{\mu'}), K(\cdot,X_{\nu'}) \rangle_{\cH_K} U_{\mu_{1}\mu_{2}\mu_{3}\nu}  \tilde{K}_{\nu\nu'}^{-1} \nonumber\\
    & = \sum_{\mu,\mu'=1}^P \sum_{\nu,\nu'=1}^P U_{\mu_{1}\mu_{2}\mu_{3}\mu}  \tilde{K}_{\mu\mu'}^{-1} K(X_{\mu'},X_{\nu'}) \tilde{K}_{\nu'\nu}^{-1} U_{\mu_{1}\mu_{2}\mu_{3}\nu}   \nonumber\\
    & = \sum_{\mu \nu=1}^P U_{\mu_{1}\mu_{2}\mu_{3}\mu} \left(\tilde{K}^{-1} K \tilde{K}^{-1}\right)_{\mu\nu} U_{\mu_{1}\mu_{2}\mu_{3}\nu} \\
    & \le \frac{\norm{\bm u}_2^2}{4\kappa}
\end{align}
where, in the final inequality, we applied Lemma~\ref{lem:tilde_K_bounds}, with $\bm{u} = (U(X_{\mu_1},X_{\mu_2},X_{\mu_3},X_{\nu}))_{\nu\in[P]} \in \bR^P$. 
Applying  Eq.~\eqref{eq:rkhs_bound}, we obtain
\begin{align}
    \abs{u_K(X_*)} < \frac{\norm{\bm u}_2}{\sqrt{4\kappa}} \sqrt{K(X_*,X_*)}=O(\sqrt{P}),
\end{align}
where we assumed properly normalized kernels and target functions. Consequently, we obtain $\tilde{U}_{\mu_{1}\mu_{2}\mu_{3}}^* = O(\sqrt{P})$, as stated, and, putting all bounds together,
\begin{align}
    \abs{\bE_{f \sim P[f]}[{f}(X_*)] - f_{K}(X_*)} \leq \frac{C P^{\gamma}}{N} + \frac{C'}{N^2},
\end{align}
for $\gamma = 5$.    
\end{proof}

%%%%%%%%%%%%%%%%%%%%%%%%%%%%%%%%
%%%%%%%%%%%%%%%%%%%%%%%%%%%%%%%%
%%%%%%%%%%%%%%%%%%%%%%%%%%%%%%%%

\section{The Fully-connected Network}
\label{app:FCN}

Consider a fully-connected network (FCN) acting on (a flattened) $X\in\bR^{T \times d}$. At the $l$-th layer, $l \geq 2$, the network computes 
\begin{align}
    z_i^{l}(X) &= b_i^{l} + \sum_{j=1}^{N_{l-1}} W_{ij}^{l} a_j^{l-1}(X), \quad a_j^{l-1}(X) = \phi\!\left(z_j^{l-1}(X)\right), \quad i \in [N_l], j \in [N_{l-1}],
\end{align}
with weights $\bm b^l \in \bR^{N_l}$ and $W^l\in\bR^{N_l \times N_{l-1}}$ drawn via
\begin{align}
    b_i^l \overset{\mathrm{i.i.d.}}{\sim} \cN(0,\sigma_{\bm b}^2), \qquad W_{ij}^l \overset{\mathrm{i.i.d.}}{\sim} \cN\!\left(0,\frac{\sigma_W^2}{N_{l-1}}\right).
\end{align}
By our convention $a^0(X)=\mathrm{vec}(X)\in\bR^{Td}$ so $N_0=Td$.
As $N_l\to\infty$, the preactivations converge in distribution to a Gaussian
process $z_i^{l}\sim\mathcal{GP}(0,K^l)$.
The complexity of evaluating the FCN NNGP kernel \citep{Lee2018DNN} for a sequence of length $T$ is dominated by the dot product scaling as $O(Td)$. By \citet{Lee2018DNN}, the recursion is initialized with
\begin{align}
    K^1(X,X') &=
    \bE\left[z_j^1(X)z_j^1(X')\right] =
    \sigma_{\bm b}^2 + \sigma_W^2
    \left(\frac{\mathrm{vec}(X) \cdot \mathrm{vec}(X')}{dT}\right).
\end{align}
with the recursion relation, for $l\ge 2$,
\begin{align}
    K^l(X,X') = \sigma_{\bm b}^2 + \sigma_W^2\, F_{\phi}
    \left(K^{l-1}(X,X'),K^{l-1}(X,X),K^{l-1}(X',X')
    \right),
\end{align}
given via a deterministic function $F_{\phi}$ whose
form depends only on the nonlinearity $\psi$. Crucially, the NNGP kernel recursions are element-wise on the Gram matrix, involving no mixing across sequence indices. Consequently, depth ($L$) only adds an additive $+L$ term to the inference-time computational complexity. Due to various fixed-points which appear in these recursions, the infinite-depth limit is also likely to be of finite computational complexity \citep{Schoenholz2016}.

%%%%%%%%%%%%%%%%%%%%%%%%%%%%%%%%
%%%%%%%%%%%%%%%%%%%%%%%%%%%%%%%%
%%%%%%%%%%%%%%%%%%%%%%%%%%%%%%%%

\section{Detailed Experiments}
\label{App:ExpDetails}

%%%%%%%%%%%%%%%%%%%%%%%%%%%%%%%%
%%%%%%%%%%%%%%%%%%%%%%%%%%%%%%%%
%%%%%%%%%%%%%%%%%%%%%%%%%%%%%%%%

\subsection{Combinatorial Tasks}
\label{app:combinatorial_tasks}

Here, we provide a detailed description of each combinatorial task used in our experiments (Sec.~\ref{sec:experiments}). We provide a symbolic formulation for each task following Sec.~\ref{sec:combinatorial_setting}. Note that different choices of score function $\psi_X$ may lead to the same underlying prediction task. Our choice of $\psi_X$ is intended to formalize a natural notion of correctness or optimality induced by the task instance $X$; throughout, we use task-structural choices rather than arbitrary post hoc encodings. As combinatorial optimization problems, SPP and MinCut are naturally combinatorial tasks with configuration sets being the set of feasible solutions and score function being the objective per task instance. In our experiments, we use those two combinatorial tasks in value form, predicting the value of the optimal solution.

\textbf{Induction.} We consider a
vocabulary of size $1024$ with sequences of length $T$. For each sequence, we first generate random tokens and select a `trigger' token from the vocabulary. To prevent ambiguity, we strictly enforce that the trigger appears only where intended by replacing any pre-existing random occurrences with the subsequent token in the vocabulary (modulo $1024$). We then inject the trigger token at a random index $i_K \in \{0, \dots, \lceil T/2 \rceil-1\}$ and again at the final sequence index $T-1$. As a symbolic combinatorial task $g:\cX\to\cY$ in solution form, we take $\cX$ to be all legal finite sequences in the vocabulary and $\cY$ to be the vocabulary set. For each sequence $X \in \cX$, the configuration set $\Conf(X)$ consists of all tokens of the vocabulary. The score function $\psi_X: \Conf(X) \to \bR$ can be chosen to be $\psi_X(\rho)=\bm{1}_{X_{i_K+1}\neq \rho}$. The tokens are then mapped to learned embeddings ($X \in \mathbb{R}^{T \times d}$) and augmented with rotary positional embeddings (RoPE) \citep{Su2021RoPE}. The objective is to predict the token following the final trigger. To solve this, the transformer must perform an induction head operation \citep{elhage2021mathematical}: it must locate the previous occurrence of the trigger token (at index $i_K$) and copy the token immediately following it ($\bm{x}_{i_K+1}$).

\textbf{Sorting Vocabulary.} We consider a vocabulary of integers $\mathcal{V}=\{0, \dots, V-1\}$ with vocabulary size $V \in \{100, 200, 300\}$. For each instance, we generate a sequence $u=(u_1,\dots,u_T)$ of length $T$, where elements are sampled uniformly with replacement from $\mathcal{V}$, and let $s=(s_1,\dots,s_T)$ denote a nondecreasing rearrangement of $u$. As a symbolic combinatorial task $g:\cX\to\cY$ in solution form, we let $\cX =\cY$ be the set of all finite sequences in the vocabulary. For each $u\in\cX$ the configuration set $\Conf(u)$ is the set of all permutations $\rho$ of $u$ and $\psi_u(\rho) = \sum_{a=1}^{T-1} (\rho_a-\rho_{a+1})_+$, where $(x)_+=\max\{x,0\}$. During training, we construct the full sequence under teacher forcing by concatenating the unsorted list $u$, a special separator token (SEP), and the sorted $s$. The inputs are mapped to learned embeddings and augmented with RoPE. The model is trained using a standard causal next-token prediction objective over the concatenated sequence $[u, \text{SEP}, s]$. However, we apply a mask to the loss function to ignore predictions on the unsorted prefix $u$, calculating the loss only on the generation of the separator and the sorted suffix $s$. Accuracy is evaluated at the token-level over the masked sorted suffix. At inference, this task requires the model to generate the sorted sequence autoregressively from the unsorted prefix: at each step of the generation phase, it is naturally interpreted as needing to attend back to the unsorted prefix $u$ to identify and copy the next smallest element that has not been generated yet.

\textbf{String Matching.} We consider random sequences of length $T$ generated from a vocabulary of size $26$. Let $M$ be a specific pattern of length $3$ in the vocabulary. The objective is to determine whether $M$ appears anywhere within the sequence. As a symbolic combinatorial task $g:\cX\to\cY$ in value form, we let $\cX$ be the set of all finite sequences in the vocabulary and $\cY=\{0,1\}$. For each $X\in\cX$, $\Conf(X)=\{X[m:m+2]\}_{m=0}^{T-3}$ and $\psi_{X}(\rho) = \bm{1}_{\rho\neq M}$. Then
\begin{align}
    g(X) = \min_{\rho \in \Conf(X)} \psi_X(\rho) = \begin{cases}
        0 & M \subseteq X \\
        1 & M \not \subseteq X
    \end{cases}.
\end{align}
We formulate the learning problem as binary classification, with the model trained to predict the label $1-g(X)$. The input for each task instance is constructed by concatenating the random sequence $X$, a special separator token (SEP), and the target pattern $M$. To prevent the model from learning trivial heuristics, we construct adversarial negative examples by ensuring the negative samples contain near-misses (e.g., matching $2$ characters of the pattern).

\textbf{Context-free Grammar (CFG) Recognition Problem.} A context-free grammar (CFG) \cite[Definition 2.2]{Sipser2006CFG} is a 4-tuple $G=(V,\Sigma,R,S)$, where $V$ is a finite set called the variables, $\Sigma$ is a finite set, disjoint from $V$, called the terminals, $R$ is a finite set of substitution rules, with each rule being a variable and a string of variables and terminals, and $S \in V$ the start variable. The language $L(G)$ of $G$ consists of all terminal strings derivable from $S$ via $R$. A parse tree of $w \in L(G)$ is one such legal derivation. We evaluate the model's ability to recognize strings in the language $L(G)$. 

We define a strict Chomsky Normal Form (CNF) grammar with $V=20$ variables and $\Sigma=16$ terminals. The branching rules consist of binary non-terminal productions ($A \rightarrow BC$) and terminal productions ($A \rightarrow a$). To maintain criticality and ensure the average derivation tree depth scales as $\mathcal{O}(\sqrt{T})$, the binary branching weights are normalized by the spectral radius of their transition matrix \cite{Chi1999}.% TODO: Replace with proper citation for critical branching processes / random trees

Crucially, our generation procedure allows us to sample strings of any arbitrary exact length $T$. While the grammar strictly enforces binary branching ($1 \rightarrow 2$) at every internal node, the derivation trees are not forced to be perfectly balanced. A parent node $A$ generating a total sequence of length $l$ can split asymmetrically into a left child $B$ that recursively generates a sub-tree of length $k$, and a right child $C$ that recursively generates a sub-tree of length $l-k$ (for any $1 \le k < l$). Both $B$ and $C$ will continue to strictly apply binary branching rules internally until they hit their respective target lengths at the terminal leaves. 

%To uniformly sample these derivation trees for an exact target length $T$ without rejection sampling, we dynamically compute the exact partition function $Z[A, L]$. Here, $A$ represents a specific non-terminal in the grammar, $L$ represents the desired string length, and $Z[A, L]$ computes the total statistical weight of all valid recursive derivations that eventually yield exactly $L$ terminals starting from $A$. During generation, if we are at non-terminal $A$ and need to produce length $L$, the probability of choosing a specific production rule $A \rightarrow BC$ and a specific length split $k$ is explicitly assigned to be proportional to $W(A \rightarrow BC) \times Z[B, k] \times Z[C, L-k]$, where $W$ is the rule's transition weight. This ensures we can efficiently and correctly sample complex, asymmetric trees of any arbitrary exact length $T$.

The dataset we generated was exactly balanced, consisting of $50\%$ strings in $L(G)$ (positive examples) and $50\%$ negative examples. To ensure the model learns robust parsing rules rather than surface-level statistics, negative examples were generated via an adversarial substitution procedure: we first sample a valid parse tree of length $T$, randomly select an intermediate length scale (to avoid power-law length bias), and pick an internal node $A$ at that scale. We then completely replace $A$'s sub-tree with a new valid sub-tree of the exact same length, but generated from an adversarial variable whose required parent production rule does not exist in the grammar.

%We define a strict critical grammar characterized by a spectral radius normalization of its binary branching rules, which maintains criticality and avoids trivial shallow derivations. To ensure the model learns robust parsing rules rather than surface-level statistics, we employ an adversarial substitution technique: negative examples are generated by taking valid parse trees and mutating an internal node such that the resulting string violates the grammar, yielding highly challenging near-misses.

\textbf{Random Geometric Graphs.}
We test algorithmic capture on two types of combinatorial graph problems. To that end, we generate random (undirected) graph instances following the Random Geometric Graph (RGG) setup \citep{Penrose2003}. First, we draw $T$ points $\bm{x}_a$, $a=1,\ldots,T$, i.i.d. from a uniform measure $\mu_{\bm{x}}$ on $[-1,1]^d$. We consider each ${\bm x_a}$ as a node in our sampled RGG. The nodes $\bm{x}_1$ and $\bm{x}_{T}$ take the special role of source and target, respectively. Next, we fix $r \geq 0$ and add an edge of weight (or capacity) $1$ between any two nodes $\bm{x}_a$ and $\bm{x}_b$ with $\|\bm{x}_a-\bm{x}_b\|_2 \leq r$. Unless stated otherwise, we parametrize $r$ via the expected degree $\alpha$ using \citet[Eq.~(7)]{Dall2002} and choose $\alpha$ near the regime where a giant linear-sized component emerges. We denote our joint input measure by $\mu_{X,T}:=\mu_{\bm{x}}^{\otimes T}$. From an average-time or heuristic complexity perspective \citep{Levin1986, Impagliazzo1995}, this distribution has low sampling complexity but generates relatively hard graph instances. 

\textbf{Shortest Path Problem (SPP).}
Consider a finite undirected graph $\cG = (\mathcal{V},\mathcal{E})$ with node set $\mathcal{V}$ of size $V=\abs{\mathcal{V}}$, edge set $\mathcal{E}$ of size $E=\abs{\mathcal{E}}$, and two distinguished nodes: a source $v_s$ and target $v_t$. A (node) path in $\mathcal{G}$ from source to target is a sequence $\pi = (v_0,\ldots,v_n)$ of nodes such that $\{v_{i-1},v_{i}\} \in \mathcal{E}$ for $i \in [n]$ beginning with $v_0=v_s$ and ending with $v_n=v_t$. Graph weighting constitutes a map $c: \mathcal{E} \rightarrow \bR_{\geq 0}$. In our RGG setting, weighting is constant $c: \mathcal{E} \rightarrow \{1\}$.  The SPP consists of finding an optimal path $\pi^*$ for
\begin{align}
    \pi^* \in \argmin_{\substack{\pi= (v_0,\ldots,v_n)\\ v_0=v_s, v_n=v_t}} \sum_{i=1}^n c\left(\{v_{i-1},v_i\}\right)
\end{align}
For our $\mu_{X,T}$, we consider the RGG at the critical connectivity threshold. The average-case time-complexity for source-target SPP is $\bE_{X\sim \mu_{X,T}}\left[t(X)\right] = \bE_{X\sim \mu_{X,T}}[E(X)+V(X)]=O(T\log(T))$, matching standard worst-case complexity for the single-pair SPP on an undirected unit-weight graph using Breadth-First Search (BFS) and critical connectivity expectation $\bE_{X\sim \mu_{X,T}}[E(X)]=\Theta(T\log T)$. As noted in Sec.~\ref{sec:setting}, this places the distributional SPP task in the EPTHS class of complexity $O(T^{1+\epsilon})$ under $\mu_{X,T}$ for every $\epsilon>0$.

\textbf{MinCut / MaxFlow.}
Starting from the RGG and graph setup above, we adapt it to add directions to edges, i.e., consider $\mathcal{E} \subseteq \mathcal{V} \times \mathcal{V}$. In a MinCut / MaxFlow network \citep[Proposition 3.2]{Bertsekas1998Network}, nodes, apart from the source and target, are termed relay nodes. In our directed graph adaptation, relay nodes become bi-directional (with equal capacity), all edges incident to the source have $v_s$ as their initial node, and all edges incident to the target have $v_t$ as their terminal node as in \citet[Definition 2]{Ramamoorthy2005Capacity}.\footnote{Equivalently, we may take the complete directed edge set $\mathcal{E} = \mathcal{V} \times \mathcal{V}$ (optionally excluding self-loops) and extend 
edge capacity $c$ by setting $c(u,v)=0$ for all directed pairs $(u,v)\in\mathcal{E}$ not present in the geometric graph; this leaves the MinCut/MaxFlow value unchanged.} A cut $\mathcal{S} \subset \mathcal{V}$ in the graph is a subset containing the source $v_s \in \mathcal{S}$ whose complement $\overline{\mathcal{S}}=\mathcal{V}\setminus\mathcal{S}$ contains the target $v_t \in \overline{\mathcal{S}}$. The MinCut problem consists of finding an optimal cut $\mathcal{S}^*$ for
\begin{align}
    \mathcal{S}^* \in \argmin_{\substack{\mathcal{S} \subset \mathcal{V} \\ v_s \in \mathcal{S}\\ v_t \in \overline{\mathcal{S}}}} \sum_{\substack{(v,\bar{v}) \in \mathcal{E}\\ v \in \mathcal{S}, \bar{v} \in \overline{\mathcal{S}}}} c\left(v,\bar{v}\right).
\end{align}
Average-case for MinCut/MaxFlow is ${O(\bE[VE])}$, matching standard worst-case complexity for MaxFlow \citep{Orlin2013MaxFlows}. For similar considerations as in SPP, this places MinCut/MaxFlow in the EPTHS class of complexity $O(T^{2+\epsilon})$ under $\mu_{X,T}$ for every $\epsilon>0$. 

%%%%%%%%%%%%%%%%%%%%%%%%%%%%%%%%
%%%%%%%%%%%%%%%%%%%%%%%%%%%%%%%%
%%%%%%%%%%%%%%%%%%%%%%%%%%%%%%%%

\subsection{Engineering}
\label{app:engineering}

We trained all models either on an Nvidia RTX 4080 or an Nvidia RTX 4090 GPU using a standard decoder-only transformer architecture (with the exception of CFG and SPP, which use a bidirectional encoder) equipped with GeLU activations. The estimated compute per experiment was  $\sim24$ hours.

For the four capture tasks Induction, Sorting Vocabulary, String Matching, and CFG, we utilized an \textit{incremental curriculum} for evaluating adaptation costs. Instead of training the model from scratch for each extrapolation horizon $T$, we pre-trained the model at a base size $T_0$ to a specific accuracy threshold $\delta$. We then iteratively adapted the model to increasing horizons, dynamically sampling sequences of length $T_{curr} \sim \mathcal{U}[T_{prev}, T_{max}]$, where $T_{max}$ was typically increased using a 20\% multiplicative schedule. The reported adaptation budget $P$ represents the \textit{cumulative sum} of samples required to maintain the $\delta$ accuracy threshold across all incremental steps. The training configurations were as follows:
\begin{itemize}
    \item \textbf{Induction}: Causal Transformer, layers $L=2$, $d=256$, heads $H=4$. We used Rotary Positional Embeddings (RoPE) with an increased base frequency ($\rho=500,000$) to facilitate stable extrapolation up to $T_{max}=10,000$. The pre-training base was $T_0=50$. Averaged over 5 seeds.
    \item \textbf{Sorting Vocabulary}: Causal Transformer, $L=2$, $d=128$, $H=2$, MLP dimension $1024$. We used RoPE with standard base $10,000$. The evaluated target threshold was $\delta=5\%$ token-level error. The pre-training base was $T_0=50$. Averaged over 3 seeds per vocabulary size. 
    \item \textbf{String Matching}: Causal Transformer, $L=3$, $d=64$, $H=1$, MLP dimension $256$. We used RoPE with base $100,000$. Models were evaluated at thresholds $\delta \in \{10\%, 5\%, 2\%\}$. The pre-training base was $T_0=50$. Averaged over 5 seeds.
    \item \textbf{CFG}: Bidirectional Transformer with mean-pooling, $L=12$, $d=128$, $H=4$. Models were evaluated at thresholds $\delta \in \{40\%, 20\%, 10\%\}$. The pre-training base was $T_0=10$. Averaged over 4 seeds.
\end{itemize}

For the combinatorial optimization tasks (SPP and MinCut), models were trained using the standard protocol on the whole sequence interval rather than the incremental protocol described above. The configurations were:
\begin{itemize}
    \item \textbf{Shallow SPP}: $L=4, d=128, B=1024$.
    \item \textbf{Deep SPP}: $L=4, d=128, B=1024$.
    \item \textbf{Shallow MinCut}: $L=4, d=128, B=1024$.
    \item \textbf{Deep MinCut}: $L=40, d=128, B=1024$.
\end{itemize}
The pre-training base was $T_0=10$ in either case. 
%%%%%%%%%%%%%%%%%%%%%%%%%%%%%%%%
% END
%%%%%%%%%%%%%%%%%%%%%%%%%%%%%%%%
\end{document}